\documentclass{article}

\usepackage[final]{neurips_2025}

\usepackage[colorlinks=true, urlcolor={blue!70!black}, linkcolor={blue!70!black}, citecolor={blue!70!black}, pdfborder={0 0 0}]{hyperref}

\usepackage[utf8]{inputenc} 
\usepackage[T1]{fontenc}    
\usepackage{url}            
\usepackage{booktabs}       
\usepackage{amsfonts}       
\usepackage{nicefrac}       
\usepackage{microtype}      
\usepackage{xcolor}         

\title{Exact Expressive Power of Transformers with Padding}

\author{
  William Merrill\thanks{Work partially conducted as a PhD student at New York University.}\\
  Allen Institute for AI\\
  \texttt{willm@allenai.org} \\
  \And
  Ashish Sabharwal \\
  Allen Institute for AI \\
  \texttt{ashishs@allenai.org} \\
}
\usepackage{amsmath}
\usepackage{amssymb}
\usepackage{amsthm}
\usepackage{thm-restate}
\usepackage{mathtools}
\usepackage{bbm}
\usepackage{xcolor}
\usepackage{url}
\usepackage{paralist}
\usepackage{wrapfig}
\usepackage{xspace}
\usepackage{enumitem}

\usepackage[capitalise]{cleveref}

\newtheorem{theorem}{Theorem}
\newtheorem{lemma}{Lemma}
\newtheorem{proposition}{Proposition}
\newtheorem{corollary}{Corollary}[theorem]
\newtheorem{lemcorollary}{Corollary}[lemma]
\theoremstyle{definition}
\newtheorem{definition}{Definition}

\definecolor{myred}{RGB}{183,53,45}

\newcommand{\draftcomment}[3]{{\textcolor{#3}{[#2: \emph{#1}]}}}

\newcommand{\AS}[1]{\draftcomment{#1}{\textsc{as}}{purple}}

\newcommand\CR[1]{{\color{blue} #1}}

\renewcommand{\draftcomment}[3]{}  
\renewcommand\CR[1]{#1}

\DeclarePairedDelimiter\abs{\lvert}{\rvert}%
\DeclarePairedDelimiter\ceil{\lceil}{\rceil}
\DeclarePairedDelimiter\floor{\lfloor}{\rfloor}

\newcommand\NC{\mathsf{NC}}
\newcommand\AC{\mathsf{AC}}
\newcommand\TC{\mathsf{TC}}
\renewcommand\L{\mathsf{L}}
\newcommand\NL{\mathsf{NL}}
\renewcommand\P{\mathsf{P}}
\newcommand\Ppoly{\mathsf{P/poly}}

\newcommand\AHAT{\mathsf{AHAT}}
\newcommand\uAHAT{\mathsf{uAHAT}}
\newcommand\mAHAT{\mathsf{mAHAT}}

\newcommand\FO{\mathsf{FO}}
\newcommand\M{\mathsf{M}}
\newcommand\bit{\mathsf{bit}}

\newcommand\FOM{\mathsf{FO}[\M, \bit]}

\renewcommand{\O}{\mathrm{O}}
\newcommand{\DLOGTIME}{\mathsf{DLOGTIME}}

\newcommand\XC{\mathsf{XC}}
\newcommand\A{\mathsf{A}}
\newcommand\B{\mathsf{B}}

\newcommand\AUniform{\A\textrm{-uniform}\xspace}
\newcommand\AUniformity{\A\textrm{-uniformity}\xspace}
\newcommand\BUniform{\B\textrm{-uniform}\xspace}
\newcommand\BUniformity{\B\textrm{-uniformity}\xspace}
\newcommand\FOUniform{\FO\textrm{-uniform}\xspace}
\newcommand\FOUniformity{\FO\textrm{-uniformity}\xspace}
\newcommand\LUniform{\L\textrm{-uniform}\xspace}
\newcommand\LUniformity{\L\textrm{-uniformity}\xspace}
\newcommand\NLUniform{\NL\textrm{-uniform}\xspace}
\newcommand\NLUniformity{\NL\textrm{-uniformity}\xspace}

\newcommand{\circuit}{\mathrm{Circuit}}
\newcommand{\gate}{\mathrm{Gate}}

\newcommand{\operator}{\mathrm{Op}}
\newcommand{\andgate}{\texttt{AND}}
\newcommand{\orgate}{\texttt{OR}}
\newcommand{\notgate}{\texttt{NOT}}
\newcommand{\majoritygate}{\texttt{MAJ}}
\newcommand{\argument}{\mathrm{Arg}}
\newcommand{\xx}{\texttt{X}}

\usepackage{tikz}
\usetikzlibrary{arrows.meta, positioning}

\usepackage{stmaryrd}
\DeclarePairedDelimiter{\den}{\llbracket}{\rrbracket}

\Crefname{figure}{Figure}{Figures}
\crefname{figure}{figure}{figures}

\bibpunct{(}{)}{;}{a}{,}{,}

\begin{document}

\maketitle

\begin{abstract}
    Chain of thought is a natural inference-time method for increasing the computational power of transformer-based large language models (LLMs), but comes at the cost of sequential decoding. Are there more efficient alternatives to expand a transformer's expressive power without adding parameters? We consider transformers with \emph{padding} tokens as a form of parallelizable test-time compute. We show that averaging-hard-attention, masked-pre-norm transformers with polynomial padding recognize precisely the class $\FOUniform\ \TC^0$ of extremely parallelizable problems. While the $\TC^0$ upper bound was known, proving a matching lower bound had been elusive. Further, our novel analysis reveals the precise expanded power of padded transformers when coupled with another form of inference-time compute, namely dynamically increasing depth via \emph{looping}. Our core technical contribution is to show how padding helps bring the notions of \emph{complete problems} and \emph{reductions}, which have been a cornerstone of classical complexity theory, to the formal study of transformers. Armed with this new tool, we prove that padded transformers with $\O(\log^d n)$ looping on inputs of length $n$ recognize exactly the class \CR{$\FO$-uniform} $\TC^d$ of moderately parallelizable problems. Thus, padding and looping together systematically expand transformers' expressive power: with polylogarithmic looping, polynomially padded transformers recognize precisely the class $\FOUniform\ \NC$, the best that could be expected without losing parallelism (unless $\NC = \P$). Our results thus motivate further exploration of padding and looping as parallelizable alternatives to chain of thought for test-time compute.
\end{abstract}

\section{Introduction}

Due to the computational limitations of transformers \citep{merrill-sabharwal-2023-parallelism,strobl-etal-2024-formal,chiang2025transformers},
solving complex reasoning problems requires
extending their computational power at inference time, typically by allowing models to generate long chains of thought (CoT) before their outputs \citep{wei2022chain,nye2022show}. Theoretical work has shown how CoT expands the expressive power of transformers to sequential problems outside the class $\TC^0$ of highly parallelizable problems, but it also sacrifices parallelism \citep{merrill2023cot,li2024chain}, making inference slow.
Are there alternative inference-time compute approaches that increase expressive power while preserving parallelism?

One method for parallelizable inference-time compute with LLMs is using \emph{padding tokens} rather than CoT.
Padding can be understood as restricted CoT where the tokens on the chain are restricted to some ``blank'' symbol rather than tokens generated by the LLM.
Since all the input tokens are known in advance, padding is more parallelizable than CoT.
There have been some attempts to make padding practical with mixed results \citep{goyal2024think,pfau-2024-think}, but it is not fully understood.
Specifically, while it is known that padded transformers remain in $\TC^0$, it has been open and elusive whether they can solve \emph{all} problems in $\TC^0$ or even the smaller class $\AC^0$ \citep{pfau-2024-think}.

Our first contribution is exactly characterizing the expressive power of transformers\footnote{\CR{Our formal model, detailed in \cref{sec:preliminaries}, assumes a fully uniform transformer with fixed parameters, fixed width, logarithmic or polynomial precision (showing their equivalence in this setting), masked pre-norm, and causal masking (\cref{thm:tc0} also applies to unmasked transformers, showing their equivalence under padding).}} with polynomial padding as $\FOUniform\ \TC^0$, answering an open question \citep{pfau-2024-think}.
The result emerges through a finer-grained analysis of transformers in terms of string logics \citep{merrill2023logic,chiang-2023-tighter,yang2024masked,yang2024counting}.
Thinking in terms of logic, we show that $n^k$ padding tokens give transformers enough ``storage space'' to resolve any first-order majority logic formula over $k$ variables.
This suffices to capture all of $\FOUniform\ \TC^0$ \citep{barrington-1990-uniformity}, giving an exact characterization of the expressive power of such transformers.

\begin{figure}[t]
    \centering

    \resizebox{0.8\textwidth}{!}{
    
    \begin{tikzpicture}[
        node distance=1cm,
        theirs/.style={->, darkgray},
        ours/.style={->, thick, myred},  
    ]
    
    \node (FOM) {$\FO+\M^2$};
    \node[above right=of FOM,xshift=0.5cm] (TC0) {$\TC^0$};
    \node[below right=of FOM,xshift=0.5cm] (AHAT0) {$\AHAT^0_*$};
    
    \draw[theirs, <->, bend right=30] (TC0) to node[above left] {\small \hypersetup{citecolor=black} \citeauthor{barrington-1990-uniformity}} (FOM);
    \draw[ours, bend right=30] (FOM) to node[below left] {\small \hypersetup{linkcolor=myred} \Cref{lem:fo-simulation}} (AHAT0);
    \draw[theirs, bend right=30] (AHAT0) to node[right,midway,rotate=-85,anchor=center,yshift=0.5cm] {\small \hypersetup{citecolor=black} \citeauthor{merrill-sabharwal-2023-parallelism}} (TC0);

    \node at (1.7,0) {\color{myred} \hypersetup{linkcolor=myred} \Cref{thm:tc0}};


    \node[right=of TC0, xshift=4.5cm] (TC) {$\TC^d, d \geq 1$};
    \node[right=of AHAT0, xshift=4.5cm] (AHAT) {$\AHAT^d_*$};
    \node[right=of FOM, xshift=3.5cm,yshift=0.6cm] (MID) {\small \color{black} \it $\FO \subseteq \AHAT^0_*$};  
    \node[below=of MID, yshift=0.3cm] (MID2) {\small \color{black} \it $\L \subseteq \AHAT^1_*$};

    \node at (8.3,0) {\color{myred} \hypersetup{linkcolor=myred} \Cref{thm:tck}};

    \draw[ours,bend right=50] (AHAT) to node[right,midway,rotate=-75,anchor=center,yshift=0.5cm] {\small \hypersetup{linkcolor=myred} \Cref{lem:transformers-in-TCd}} (TC);
    \draw[ours,-,bend right=25] (TC) to node[left,xshift=-0.5cm,yshift=-0.1cm] {\small \hypersetup{linkcolor=myred} \Cref{lem:fo-simulation}} (MID);
    \draw[ours=30,-,bend right=10] (MID) to node[left] {\small \hypersetup{linkcolor=myred} \Cref{lem:nl-lb}} (MID2);
    \draw[ours,bend right=30] (MID2) to node[left,xshift=-0.7cm,yshift=0.2cm] {\small \hypersetup{linkcolor=myred} \Cref{lem:TCd-in-transformers}} (AHAT);
    
    \end{tikzpicture}
    }

    \caption{Summary of core results: exact characterizations of the expressive power of  $O(\log^d n)$-depth looped AHATs with padding, for $d \geq 0$.
    \Cref{thm:tc0} shows that $\AHAT^0_* = \FOUniform\ \TC^0$.
    \Cref{thm:tck} extends this to show that, for $d \geq 0$, $\AHAT^d_* = \FOUniform\ \TC^d$.
    \CR{In the process of obtaining these results, we also found the novel circuit complexity result that, for any $d \geq 1$, $\FOUniform\ \TC^d = \LUniform\ \TC^d$ \Cref{thm:uniformity-collapse-ACd-TCd}.
    Thus, for $d \geq 1$, $\AHAT^d_* = \LUniform\ \TC^d$.
    }
    }
    \label{fig:results-summary}

    \vspace{-0.5em}

\end{figure}

This first result, however, does not clarify whether padded transformers can be minimally extended to gain expressivity \emph{beyond} $\TC^0$.
To address this, we consider the combination of padding with \emph{looping}, i.e., repeating a block of layers dynamically as a function of input length, \CR{also referred to as \emph{universal} transformers} \citep{dehghani2018universal,giannou-2023-looped}.
This can be understood as a form of inference-time compute where padding controls the computation width and looping controls the computation depth---crucially, without adding any parameters.
If the number of repetitions is minimal (e.g., sublinear) in sequence length, this transformer model remains highly parallelizable relative to CoT \citep{merrill-2025-little}.

Extending this result, we show that, for $d \geq 1$, transformers
with polynomial padding and $\O(\log^d n)$ looping recognize exactly \CR{$\FOUniform\ \TC^d$}.
Thus log or polylog-looped transformers can solve many problems outside $\TC^0$ under standard complexity conjectures, including boolean formula evaluation, iterated matrix multiplication, graph connectivity, and context-free language recognition.
It also follows that polylog-looped padded transformers converge in expressive power to \CR{$\FOUniform$} $\NC$, the ceiling that could be expected while preserving parallelism under standard complexity conjectures.

Our first result about the expressive power of fixed-depth padded transformers provides crucial insight for our extended results about looped padded transformers (cf.~\Cref{fig:results-summary}), by allowing, for the first time, the use of familiar tools of \emph{reductions} and \emph{complete problems} from classical complexity theory in the analysis of transformers.
First, the ability to express $\TC^0$ implies that padded transformers can implement $\FO$ reductions.
We then scaffold the ability to implement $\FO$ reductions to show that transformers can also implement $\L$ reductions, via known results about transformers' ability to implement graph connectivity, an $\NL$-complete problem under $\FO$ reductions.
We then use this to show that $\O(\log^d n)$ looping recognizes exactly \CR{$\FOUniform\ \TC^d$}, via the \CR{``wide''} $\TC^d$ circuit evaluation problem, which we prove to be complete for this class under $\L$ reductions.

\CR{Finally, we prove a novel circuit complexity result that $\FOUniform\ \TC^d = \LUniform\ \TC^d$ for any $d \geq 1$.
This was obtained in our analysis of looped transformers with $d \geq 1$.
Besides being potentially of independent interest in circuit complexity, this allows us to see that, for $d \geq 1$, $\O(\log^d n)$ looping allows recognizing not just $\FOUniform\  \TC^d$ but also $\LUniform\ \TC^d$ (because they are the same).}



Overall, our results provide an \emph{exact characterization} of the expressive power of padded fixed-depth transformers as $\TC^0$, answering an open question \citep{pfau-2024-think}.
Further, we exactly characterize the expressive power of polylog-looped padded transformers, which reveals that padding and looping dramatically extend the expressive power of transformers under standard complexity conjectures.
We take these results to motivate empirical investigation of padding and looping as forms of inference-time compute that are more parallelizable than standard CoT.

\section{Preliminaries} \label{sec:preliminaries}

\subsection{Averaging-Hard-Attention Transformers}

Following previous work, we analyze a model of transformers that is slightly idealized compared to standard soft-attention transformers.
Specifically, we analyze transformers with averaging hard attention (AHAT; also called ``saturated'' by \citealp{merrill-etal-2022-saturated}) and masked pre-norm \citep{merrill2023cot}.
AHAT attention heads only attend to the values that \emph{maximize} the attention score; in the case of ties, the head returns a uniform average of all values with tied scores.
Masked pre-norm means that the transformer can apply a linear projection before pre-norm at the beginning of each sublayer, which is useful for reading individual values from the residual stream without interference from other values.
Overall, we believe these idealizations are minimal changes that make it easier to implement algorithmic constructions that generalize to any sequence length.

More formally, we define the transformer sublayers in this AHAT model following \citet{merrill-2025-little}.
Each sublayer uses masked pre-norm \citep{xiong-2020-layer,merrill2023cot}, reading input $\mathbf z_i = \mathsf{layer\_norm}(\mathbf M \mathbf h_i)$,
where $\mathbf h_i$ is the previous sublayer output, $\mathbf M$ is some matrix, and layer-norm can be standard layer-norm \citep{ba2016layernormalization} or RMS norm \citep{zhang2019rms}.
The sublayer outputs $\delta_1, \ldots, \delta_n$, and the residual stream is updated as $\mathbf h'_i = \mathbf h_i + \delta_i$.

\begin{definition}[Self-attention sublayer]
    The self-attention sublayer is parameterized by a mask $\mathbf m \in \mathbb Q^m$, output projection matrix $\mathbf W : \mathbb Q^m \to \mathbb Q^m$, and, for $1 \leq k \leq h$, query, key, and value matrices $\mathbf Q^k, \mathbf K^k, \mathbf V^k$, each of which is a projection from $\mathbb Q^m$ to $\mathbb Q^{m / h}$.
\end{definition}

Given input $\mathbf z_i$, the self-attention sublayer computes queries $\mathbf q_i = \mathbf Q^k \mathbf z_i$, 
keys $\mathbf k_i = \mathbf K^k \mathbf z_i$, 
and values $\mathbf v_i = \mathbf V^k \mathbf z_i $.
Next, these values are used to compute the attention head outputs:
\begin{equation*}
    \mathbf a_{i,k} = \lim_{\tau \to 0} \sum_{j=1}^{c} \frac{\exp(1/\tau \cdot \mathbf q_{i,k}^\top \mathbf k_{j,k})}{Z_{i,k}} \cdot \mathbf v_{j,k} ,
    \quad \textrm{where} \;
    Z_{i,k} = \sum_{j=1}^{c} \exp \left(1/\tau \cdot \mathbf q_{i,k}^\top \mathbf k_{j,k} \right) .
\end{equation*}
We can set $c = i$ to define \emph{causally masked} attention and $c = n$ for \emph{unmasked} attention.
Averaging hard attention is formalized by taking the low-temperature limit ($\tau \to 0$), which causes all probability mass to be concentrated on the tokens that maximize the attention score.
In practice, transformers can approximate this by learning a temperature close to zero; for a fixed sequence length, this approximation will hold, but for longer strings it will break down.
Finally, the output of the self-attention sublayer is computed by aggregating the head outputs via $\delta_i = \mathbf W \cdot \mathrm{concat}(\mathbf a_{i,1}, \ldots, \mathbf a_{i,h})$.
    
\begin{definition}[Feedforward sublayer]
    The feedforward sublayer at layer $\ell$ is parameterized by a mask $\mathbf m \in \mathbb Q^m$ and projections $\mathbf W: \mathbb Q^m \to \mathbb Q^w$ and $\mathbf U : \mathbb Q^w \to \mathbb Q^m$.
\end{definition}

A feedforward layer computes a local update to the residual stream via $\delta_i = \mathbf U \cdot \mathsf{ReLU}(\mathbf W \mathbf z_i)$.

A transformer defines a function $\Sigma^* \to \Sigma$ or $\Sigma^* \to \{0, 1\}$ if we add a linear head to the final layer and take its argmax as the output.
We say that a transformer $T : \Sigma \to \{0, 1\}$ recognizes a language $L$ (with beginning-of-sequence token $\$$) if, for any $w \in \Sigma^*$, $T(\$w) = 1$ if and only if $w \in L$.

\paragraph{Precision.}
\CR{In \Cref{sec:datatype}, we formalize logarithmic \citep{merrill-sabharwal-2023-parallelism} and polynomial precision \citep{chiang2025transformers} datatypes.
All our constructions go through with either datatype, showing that logarithmic and polynomial-precision looped padded transformer classes are, in fact, identical.}

\paragraph{Layer-Norm Hash.} Our constructions will use the layer-norm hash representation \citep{merrill2023cot,merrill-2025-little} for query-key matching with positive integer values.

\begin{definition} \label{def:layer-norm-hash}
    Given $z \in \mathbb N$, its \emph{layer-norm hash} is the vector $\langle z, 1, -z, -1 \rangle / \sqrt{2z^2 + 2}$ .
\end{definition}

The layer-norm hash is computable within a transformer and satisfies that property that $\phi(i)^\top \cdot \phi(j) = 1$ if and only if $i = j$. This makes it useful for retrieval based on position matching using AHATs.

\subsection{Padded and Looped AHATs}

We assume a looped transformer model as previously defined by \citet{merrill-2025-little}:

\begin{definition}[$d(n)$-looped transformer]
    A looped transformer's layers are partitioned into blocks $\langle A, B, C \rangle$. On an input of length $n$, block $B$ is repeated depth-wise $d(n)$ times.
\end{definition}

Looped transformers provide a way of dynamically scaling width at test time by looping a block of layers. In addition, we also consider \emph{padding} \citep{pfau-2024-think} as a simple way to increase width:

\begin{definition}[$w(n)$-padded transformer]
    On an input of length $n$, we first append $w(n)$ ``blank'' tokens ($\square \not\in \Sigma$) and then feed this full string through the transformer.
\end{definition}

Looping and padding can also be combined, giving us the following language class for transformers:

\begin{definition}[Padded and looped transformers] \label{def:T_d_k}
    Let $d, k \in \mathbb{Z}_{\geq 0}$. $\AHAT^d_k$ is the class of languages recognizable by a causally masked looped transformer with masked pre-norm, a beginning-of-sequence token, no position embedding, $\O(\log^d n)$ depth, and $\O(n^k)$ padding tokens. Further, $\AHAT^d_* = \bigcup_{k=0}^\infty \AHAT^d_k$, $\AHAT^*_k = \bigcup_{d=0}^\infty \AHAT^d_k$, and $\AHAT^*_* = \bigcup_{d=0}^\infty \AHAT^d_*$.
\end{definition}

$\AHAT^d_k$ uses causal masking with no positional encodings (but with a beginning-of-sequence (BoS) token $\$$; cf.~\citealp{merrill2023cot}).
\CR{In contrast, some of our constructions will be for \textbf{unmasked transformers}, which must use position encodings to distinguish positions.
We will write $\uAHAT^d_k$ for the language classes recognizable by unmasked transformers with $1/i$ position encodings.
We will occasionally consider \textbf{mixed-masked transformers} where some attention heads use causal masking and others do not.
These transformers, which do not need position encodings since they can compute $1/i$ with causal heads and the BoS token, will be denoted $\mAHAT^d_k$.}

With some abuse of notation, we will use $\AHAT^d_k$ to denote both the class of languages as defined above as well as the corresponding class of transformer models, i.e., transformers with $\O(\log^d n)$ and $(n^k)$ padding tokens. The distinction should be clear from the context.

\subsection{Circuit Complexity}

We define the circuit complexity classes $\AC^d$ and $\TC^d$ in the standard way \citep{arora-2009-computational,strobl-etal-2024-formal}.
A circuit family is a mapping $\mathcal C = \{ C_n \}_{n=0}^\infty$ from input lengths $n$ to circuits that take $n$ inputs.
$\AC^d$ is the class of language that can be recognized by polynomial-size, fixed-depth circuit families with unbounded-arity AND/OR gates.
$\TC^d$ is the same class but augmented with unbounded-arity MAJ gates that take a majority vote over their input bits.

It will often be useful to talk about \emph{uniform} variants of these classes. A $\mathsf X$-uniform circuit family obeys that constraint that circuit $C_n$ for input size $n$ can be ``built'' from the input string $1^n$ using computation in $\mathsf X$.
We will consider two standard notions of uniformity: $\FO$ uniformity (which is equivalent to $\DLOGTIME$ uniformity; \citealp{barrington-1990-uniformity}) and the weaker $\L$ uniformity (i.e., log-space uniformity).
\CR{For a circuit class $\XC$, we will write $\A$-uniform $\XC$ to denote that class constrained to circuit families satisfying $\A$ uniformity.
See \cref{appendix:uniformity} for formal definitions of each type of uniformity and \citet{strobl-etal-2024-formal} for further context.}

Finally, we will also use the notion of completeness for circuit classes. Informally, an $\mathsf X$-complete problem is a problem in $\mathsf X$ to which any other problem in $\mathsf X$ can be mapped via some simple reduction $\mathsf R$. Whether a problem is complete depends on the notion of reduction used.
See \Cref{def:reduction} for a more formal definition of reductions, which makes completeness fully defined.

\subsection{Logic}

We define the standard notion first-order logic over strings ($\FO$; cf.~\citealp{merrill2023logic}).
$\FO$ formulas map strings to boolean values.
Its formulas can check for token occurrences at specific string positions, and it allows quantification over positions in the input string.
More formally:

\begin{definition} \label{def:fo}
    $\FO$ contains two types: indices, representing positions in the input string, and formulas, which evaluate to true or false.
    For an index or formula $x$, we write $\den{x}^{w,v}$ for the evaluation of $x$ on string $w$ with variable assignments $v$ (a map from names to values).
    Indices in $\FO$ are integers denoting positions in the input string:
    \begin{compactenum}
        \item The constant $1$, representing the first token's position: $\den{1}^{w,v} = 1$.
        \item The constant $n$, representing the last token's position: $\den{n}^{w,v} = \abs{w}$.
        \item Symbols (e.g., $i, j, k$) representing variables ranging over positions $1$ to $n$: $\den{i}^{w,v} = v[i]$.
    \end{compactenum}
    Formulas in $\FO$ are then constructed as follows:
    \begin{compactenum}
        \item Let $\Sigma$ be a finite alphabet. For each $\sigma \in \Sigma$ and any index $i$, $Q_\sigma(i)$ is a formula that is true when the $i$-th input token is $\sigma$. That is, $\den{Q_\sigma(i)}^{w,v} = 1$ iff $w_m = \sigma$ where $m = \den{i}^{w,v}$.
        \item For two indices $i, j$, $i = j$, $i \leq j$, and $i \geq j$ are formulas with their conventional semantics.
        \item For two indices $i, j$, $\bit(i, j)$ is a formula returning the $j$-th bit of $i$.
        \item For two formulas $P, Q$, $P \wedge Q$ and $P \vee Q$ are formulas with their conventional semantics.
        \item For any formula $P$ (which may refer to $i$ or any over variable), the following are formulas:
        \begin{compactenum}
            \item $\exists i . P$ means setting $i$ to some value $m \in [1, n]$ makes $P$  true (more formally, $\den{P}^{w,v\mid i=m} = 1$).
            \item $\forall i . P$ means setting $i$ to any value $m \in [1, n]$ make $\phi$ true.
        \end{compactenum}
    \end{compactenum}
\end{definition}

An $\FO$ formula $P$ with no free variables is called a \emph{sentence} and returns a value in $\{0, 1\}$ for each input string. The language \emph{defined} by a sentence is the set of strings mapped to 1.
The \emph{nesting depth} of an $\FO$ formula is the depth of its syntactic tree constructed by the rules above. The \emph{number of distinct variables} is the number of different variable names used in $P$. This can be minimized by allowing two independent quantifiers in parallel subformulas to use the same variable name.

It is known that the class of languages defined by $\FO$ is exactly the circuit complexity class $\DLOGTIME\textrm{-uniform}\ \AC^0 = \FO\textrm{-uniform}\ \AC^0$ \citep{barrington-1990-uniformity}.
This class captures the languages recognized by unique hard-attention transformers \citep{hao-etal-2022-formal,yang2024masked}, but it is not large enough to capture soft-attention transformers \citep{merrill-sabharwal-2023-parallelism}.
To consider a more expressive logic capable of modeling soft-attention transformers, we can extend $\FO$ with majority quantifiers \citep{merrill2023logic}.
Specifically, we will define $\FO + \M^2$ as $\FO$ extended to include a \emph{paired majority} quantifier:

\begin{definition}[\citealp{barrington-1990-uniformity}] \label{def:fom2}
    $\FO + \M^2$ is $\FO$ extended to include the $\M^2$ quantifier, defined as follows: $\M^2 (i, j) . P(i, j)$ is true if $\phi(i, j)$ holds for a majority of \emph{pairs} of positions $(i, j) \in [n]^2$ in the string.
\end{definition}

It is known that $\FO + \M^2$ defines exactly $\FOUniform\ \TC^0$ \citep[\CR{Theorem 10.2, which uses paired majority}]{barrington-1990-uniformity}.
It is also possible to use majority quantifiers to define addition over indices, so without loss of generality, we can assume $\FO + \M^2$ formulas have no addition (unlike $\FO$ formulas).
Moreover, it is possible to simulate $\bit$ in terms of $\M^2$, so we can consider $\FO + \M^2$ formulas not to contain $\bit$, in contrast to the more standard logic $\FO + \M[\bit]$ that also defines $\TC^0$.
This makes $\FO + \M^2$ a simpler target for our transformer constructions than $\FO + \M[\bit]$.

\section{Masked and Unmasked Transformers}
\label{sec:masking}

Before presenting our main results, we briefly discuss a simple but important relationship between different kinds of masking that will come in handy.
When using a transformer as a recognizer for formal languages, there is a choice of whether to use an encoder (with no attention masking) or a decoder (with causal attention masking) to encode the input string.
Causally masked (decoder) models are more standard these days, as well as theoretically more convenient (e.g., no need for position embedding, as position information can be derived from causal attention), so we take them as our default model.
However, it is sometimes easier to reason about how to solve problems with unmasked transformers.
Fortunately, the following lemma shows we can simulate an unmasked transformer (encoder) with a causally masked transformer (decoder), if we allow padding tokens proportional to the transformer depth.
This will be useful going forward in several places where we convert unmasked and mixed-masked constructions to causally masked constructions.

\begin{lemma}[Unmasked to Causally Masked]\label{lem:mask-conversion}
    Let $E$ be an unmasked \CR{(with position encoding $1/i$)} AHAT encoder with depth $\ell \geq 1$.
    Then there exists a causally masked AHAT decoder $D$ \CR{(without any position encoding)} with depth $\ell$ and with $\ell n$ padding tokens on input length $n$ that is equivalent to $E$ in the following sense: the final $n$ outputs of $D$ match the original $n$ outputs of $E$.
\end{lemma}

\begin{proof}
    \CR{We first observe that the causally masked decoder $D$ can compute $1/i$ (the position encoding used in the unmasked encoder $E$) by attending uniformly with value $1$ only for the beginning-of-sequence symbol.
    To simulate unmasked attention with causally masked attention via padding,} the idea is for $D$ to unroll a sequence of $\ell$ ``blocks'' \CR{(one for each of the $\ell$ layers of $E$)} along the padding length dimension, each of width $n$. Each block will attend to the previous block in the previous layer and can thus see all tokens despite causal masking.

    To implement the block construction, we first compute $\phi(n)$. We then use this to compute $\phi(b_i)$, where $b_i = \floor{i / n} \in [1, \ell]$, which represents the block that each token belongs to.
    We also compute $\phi(b_i - 1) \in [0, \ell-1]$ using a head that ignores the beginning-of-sequence token $\$$.
    Next, we modify each original head from $U$ to have an additional term in the query/key product. The query is $C \phi(b_i - 1)$ and the key is $\phi(b_j)$, where $C$ is a large weight. We set $C$ to a fixed large value (independent of $n$) such that this term dominates all other terms in the inner product computation at each token.
    As a result, the head is constrained over keys in the context where $b_j = b_i - 1$, i.e., the keys from the last block.
    Within these keys, it computes exactly the same AHAT output as the original unmasked head.
    In this way, we are able to simulate an unmasked (or mixed mask) transformer with a causally masked transformer.
\end{proof}

\CR{Since causally masked heads of a mixed-masked transformer can be trivially simulated by a causally masked transformer, \cref{lem:mask-conversion} allows us to relate various masking variants (proof in~\Cref{sec:proofs}):}

\begin{restatable}{proposition}{maskconversion}\label{prop:mask-conversion}
    Unmasked \CR{(with position encoding $1/i$ or $i/n$; cf.~\cref{lem:masked-pos} in \S\ref{sec:proofs})}
    and mixed-masked padded transformers (i.e., uAHATs and mAHATs) can be simulated by causally masked transformers (AHATs). Specifically, for any $d, k \in \mathbb{N}$, the following holds for the corresponding problem classes:
    \begin{compactenum}
        \item $\uAHAT^0_k \subseteq \mAHAT^0_k \subseteq \AHAT^0_{\max\{k, 1\}}$
        \item $\uAHAT^d_k \subseteq \mAHAT^d_k \subseteq \AHAT^d_{1 + \max\{k, 1\}}$ for $d \geq 1$
        \item $\uAHAT^d_* \subseteq \mAHAT^d_* \subseteq \AHAT^d_*$.
    \end{compactenum}
\end{restatable}

\section{Fixed-Depth Padded Transformers Recognize Exactly \texorpdfstring{$\FOUniform\ \TC^0$}{FO-Uniform TC0}}



We are now ready to prove our first main result, namely that padding tokens allow transformers to simulate $\FO$ formulas---with more padding allowing more nesting of variables.
Moreover, they can simulate formulas with the special \emph{paired majority} quantifiers.

\begin{lemma} \label{lem:fo-simulation}
    An $\FO + \M^2$ formula with $k \geq 1$ distinct variables and nesting depth $\ell$ can be computed in $\uAHAT^0_k$ (and hence in $\AHAT^0_k$) with (fixed) depth $\ell$.
\end{lemma}

\begin{proof}

    We will store all $n^k$ possible configurations of $k$ variables of the formula using $n^k$ padding tokens, where each token corresponds to a specific configuration of all the variables, which we denote $v$ (cf.~\Cref{def:fo}).
    We will present an inductive transformer construction that uses a single layer to compute the boolean variable of each formula and the integer value of each numerical expression in $\FO + \M^2$, assuming the constituent formulas and values were already computed at the previous layer.
    Since each $\FO + \M^2$ formula consists of a fixed number of subformulas and values, we can accumulate all constituents in the residual stream in order to compute any larger formula.

    In more detail, let $\den{x}^{w,v}$ be the value of a formula or numerical value $x$ on string $w$ under assignment $v$, with $k$ total variables.
    We will identify each variable assignment with an integer $v \in [n^k]$,
    so that, for every $P$, padding token $v$ stores $\den{P}^{w,v}$ in the residual stream with a scalar whose sign indicates a truth value. For a numerical value $i$, token $v$ will represent $\den{i}^{w,v}$ as a small vector $\phi(\den{i}^{w,v})$. where $\phi$ is the layer-norm hash (\Cref{def:layer-norm-hash}).
    We show how to evaluate each constituent of an $\FO + \M^2$ formula.
    As mentioned after \Cref{def:fom2}, addition over indices and $\bit$ are subsumed by $\M^2$, so we do not need to simulate them.
    
    \begin{enumerate}[leftmargin=*]
        \item \underline{Constants.} To compute the constant $1$ or $n$ at each configuration $v$, we can attend from each $v$ to the first and last input token and retrieve its position to get $\phi(1)$ or $\phi(n)$, where $n = \abs{w}$.

        \item \underline{Variables.} At each token $v$, we will compute $\phi(\den{i}^{w,v})$, a representation of the value of variable $i$ under assignment $v$.
        We view the integer $v \in [n^k]$ as a tuple of $k$ values $v[1], \ldots, v[k] \in [n]$. We have that $\den{i}^{w,v} = v[i]$. To get this, we first compute $\phi(v)$ and then take a ``projection'' to retrieve $\phi(v[i])$ using quotient and remainder operations \citep[Lemma 3.1]{merrill-2025-little}. More formally, we compute $\phi(n)$ where $n = \abs{w}$, divide $i - 1$ times by $n$, and then take the remainder $\bmod\ n$ to obtain:
        \begin{equation*}
            \phi(\floor{v / n^{i-1}} \bmod n) = \phi(v[i]) .
        \end{equation*}
        Thus, we conclude that we can compute $\phi(v[i]) = \phi(\den{i}^{w,v})$ at each token $v$.

        \item \underline{Comparisons.} Given two numerical values $i, j$ (either variables or constants) already stored at $v$ as $\phi(\den{i}^{w,v}), \phi(\den{j}^{w,v})$, we can simply compare $\phi(\den{i}^{w,v}) - \phi(\den{j}^{w,v})$ on some axis and apply layer-norm to obtain $\den{i > j}^{w,v}$ at token $v$.
    
        \item \underline{Token Predicates.} Assume we have $\phi(\den{i}^{w,v})$ previously stored at assignment $v$. We can hard-attend to retrieve token $w_m$ where $m = \den{i}^{w,v}$. We then compute $\den{Q_\sigma(i)}^{w,v}$ by checking whether $w_m = \sigma$.

        \item \underline{Connectives.} Assume we have $\den{P}^{w,v}$ and $\den{Q}^{w,v}$ previously stored at each assignment token $v$.
        Then we can simply use a feedforward network to compute $\den{\neg P}^{w,v}$, $\den{P \wedge Q}^{w,v}$, or $\den{P \vee Q}^{w,v}$ independently at each $v$.

        \item \underline{Standard Quantifiers.} Assume we have $\den{P}^{w,v}$ stored at each configuration $v$.
        Then, at each $m$, we want to resolve the quantifier $\mathsf Q$ over the set $\{ \den{P}^{w,v|i=m}\}_{m=1}^n$, where $v|i=m$ denotes $v$ with $i$ overriden to have value $m$. We will count $c$, the number of $m$ such that $P^{v|i=m}$ holds. More formally, let $j_1, \ldots, j_{k-1}$ be the set of variables excluding $i$. Then we can attend from $v$ over $v'$ with query $\langle \phi(j_1^v), \ldots, \phi(j_{k-1}^v) \rangle$, key $\langle \phi(j_1^{v'}), \ldots, \phi(j_{k-1}^{v'}) \rangle$, and value $P^{v'}$ to retrieve $c/n$.
        Finally, we threshold $c/n$ against $\frac{1}{2n}$ (for $\exists$) or $\frac{2n - 1}{2n}$ (for $\forall$) to resolve $\den{\mathsf{Q} i.\phi}^{w,v}$.

        \item \underline{Paired Majority Quantifiers.}
        Assume we have $\den{P}^{w,v}$ stored at each configuration $v$.
        We want to compute $\den{\mathsf{M}(i,j).P}^{w,v}$ for any two variables $i, j$ already represented.
        The idea slightly generalizes the construction for single-variable quantifiers: we will use attention to count $c$, the number of assignments where $\den{P}^{w,v|i=m, j=\ell}$ is true.
        Formally, let $j_1, \ldots, j_{k-2}$ be the set of variables excluding $i, j$. Then we can attend from $v$ over $v'$ with query $\langle \phi(j_1^v), \ldots, \phi(j_{k-2}^v) \rangle$, key $\langle \phi(j_1^{v'}), \ldots, \phi(j_{k-2}^{v'}) \rangle$, and value $\den{P}^{w,v'}$ to retrieve $c/n^2$.
        Finally, we threshold $c/n^2$ against $\frac{1}{2}$ to resolve $\den{\mathsf{M}^2(i,j).P}^{w,v}$.
    \end{enumerate}

    In conclusion, we can compute any $\FO + \M^2$ formula in $\uAHAT^0_k$ by inductively computing its constituent formulas and storing their values over $k$ variables using $n^k$ padding tokens.
    The result extends to $\AHAT^0_k$ by applying \Cref{prop:mask-conversion}.
\end{proof}





Our construction for \Cref{lem:fo-simulation} somewhat resembles (though is distinct from) a result of \citet[Corollary 6.8]{lange-2004-results} that any problem in $\TC^0$ can be reduced to majority logic via a transformation that appends some extra tokens to the input.
However, their transformation is not padding since it appends some ``non-blank'' tokens, and their result also does not clearly apply to transformers.

Combined with previous results about transformers being in $\FOUniform\ \TC^0$, \Cref{lem:fo-simulation} yields an exact characterization of constant-depth transformers with padding:




\begin{theorem} \label{thm:tc0}
    $\uAHAT^0_* = \AHAT^0_* = \FOUniform\ \TC^0$, i.e., $\FOM$.
\end{theorem}

\begin{proof}
    It is known that $\AHAT^0 \subseteq \FOUniform\ \TC^0$ \citep{merrill-sabharwal-2023-parallelism,merrill2023logic,chiang2025transformers}, and this generalizes to transformers with padding tokens \citep{pfau-2024-think}.
    We will show $\FOUniform\ \TC^0 \subseteq \uAHAT^0_*$.
    \citet[Proposition 10.3]{barrington-1990-uniformity} proved that $\FOUniform\ \TC^0$ is definable by $\FO$ formulas with $\M^2$ quantifiers: notably, the $\bit$ predicate is not necessary when using $\M^2$ quantifiers.
    Thus, \Cref{lem:fo-simulation} establishes that $\FOUniform\ \TC^0 \subseteq \uAHAT^0_*$.
    Finally, from \Cref{prop:mask-conversion}, we have $\uAHAT^0_* \subseteq \AHAT^0_*$.
    Hence, $\uAHAT^0_* = \AHAT^0_* = \FOUniform\ \TC^0$.
\end{proof}

A technical hurdle to obtaining this characterization in prior work was simulating the $\bit$ predicate in standard definitions of $\TC^0$ \citep{pfau-2024-think}. Our results circumvent this by instead relating transformers to $\FO + \M^2$, which is equivalent to standard $\FOM$ \citep{barrington-1990-uniformity}.
It thus follows from \cref{thm:tc0} that padded transformers can also simulate $\bit$ as well as all of $\FO$, which will be useful in the following section for simulating $\FO$ reductions.

\CR{Lastly, we note that the arguments leading to \cref{thm:tc0} work for both logarithmic and polynomial precision (cf.~\cref{sec:datatype}) transformers. This shows that these two levels of precision lead to identical power for fixed-depth padded transformers.}



\section{\texorpdfstring{Log$^d$-}{}Looped Padded Transformers Recognize Exactly \texorpdfstring{$\FOUniform\ \TC^d$}{FO-Uniform TCd}}

The notion of \emph{completeness} of a problem (or language) for a complexity class under certain types of \emph{reduction} has played a key role in computational complexity. Just like it has been useful for reasoning about the expressivity of resource-bounded standard models of computation (Turing machines, circuit models, etc.), we show it can also be used to reason about the expressivity of padded transformers.

We begin by formally defining the notion of reductions in terms of predicates or languages (rather than string-to-string functions).
This will make it easier to precisely implement reductions inside transformers, which produce contextual representations of prefixes of the input string in parallel, in contrast to the more standard definition of reductions as string-to-string functions. 

\begin{definition}\label{def:reduction}
    Let $b(i)$ be the binary encoding of $i \in \mathbb N$ in some alphabet $\Sigma \supseteq \{0, 1\}$.
    Let $\mathsf R$ be a class of languages.
    We say a transduction $f : \Sigma^* \to \Sigma^*$ is an \emph{$\mathsf R$ reduction} if
    $\abs{f(w)}$ is polynomial in $\abs{w}$ and
    the language $R_f = \{ (w, b(i), \sigma) \mid f(w)_i = \sigma \}$ is in $\mathsf R$.\footnote{Here $f(w)_i$ denotes the $i$-th token of $f(w)$ if $i \leq \abs{f(w)}$, and a special symbol $\square \not \in \Sigma$ otherwise.}
\end{definition}

\Cref{def:reduction} recovers the standard notions of $\FO$ and $\L$ reductions.
A transformer can be said to compute a reduction $f$ if it recognizes the language of triples $(w, i, \sigma)$ defined above. Since $\abs{\Sigma}$ is finite and transformers can equally easily output a `token' in $\Sigma$ instead of just 0/1, it is in fact more natural to require the transformer to compute a functional form of this language, namely compute $r_f(w, i)$ defined as $f(w)_i$. Our constructions work under both views, though the latter is often more natural and efficient. Formally:

\begin{definition}\label{def:transformer-computes-reduction}
We say that a \emph{transformer computes an $\mathsf R$ reduction} $f$ if it either recognizes the language $R_f = \{ (w, i, \sigma) \mid f(w)_i = \sigma \}$ in $\mathsf R$ or computes the function $r : \Sigma \times \mathbb{N} \to \Sigma$ defined as $r_f(w, i) = f(w)_i$, where, in either case, $i$ is encoded in binary.
\end{definition}


\begin{restatable}{lemma}{ReductionsForTransformers} \label{lem:reductions}
    Let $\mathsf C, \mathsf R$ be classes of languages. Let language $L$ be $\mathsf C$-complete under $\mathsf R$ reductions.
    If $\AHAT^d_*$ transformers can recognize $L$ and compute every $\mathsf R$ reduction,
    then $\mathsf C \subseteq \AHAT^d_*$.
\end{restatable}

\begin{proof}[Proof Sketch]
    Consider any language $L' \in \mathsf C$. By the assumed completeness of $L$, there exists an $\mathsf R$ reduction $f$ that maps inputs of $L'$ into inputs of $L$ such that $w \in L'$ if and only if $f(w) \in L$. From the last precondition of the theorem, there exists a causally masked log$^d$-depth padded transformer $T_f$ that recognizes the corresponding reduction language $R_f$ or, equivalently, computes the reduction function $r_f$ (cf.~\Cref{def:transformer-computes-reduction}). We will assume the latter, i.e., that $T_f$ computes $r_f$, though the construction can also be made to work if $T_f$ checks membership in $R_f$. Additionally, we also have from a precondition that there exists a causally masked log$^d$-depth padded transformer $T_L$ that recognizes $L$.

    The idea is to ``stack'' $T_L$ on top of $T_f$ to obtain a log$^d$-depth padded transformer $T$ that recognizes $L'$. Intuitively, given an input $w$, the first set of layers of $T$ will compute $f(w)$ tokenwise, by computing $r(w, i) = f(w)_i$ for every $i$ in parallel. To this end, we will essentially make $\abs{f(w)}$ copies of the padding tokens needed by $T_f$ and perform the computation of $T_f$ independently for each $f(w)_i$. The second and final set of layers of $T$ will then check whether the string $f(w)$ produced by the first set of layers is in $L$, which will hold if and only if $w \in L'$. \CR{See full proof in \cref{sec:proofs}.}
\end{proof}

Combined with results in prior work, it follows from \Cref{lem:reductions} (see proof below) that padded log-depth transformers can recognize any language in $\NL$:

\begin{lemma} \label{lem:nl-lb}
    $\NL \subseteq \AHAT^1_*$.
\end{lemma}

\begin{proof}
    Let $L$ be the graph connectivity problem, class $\mathsf C$ be $\NL$, and class $\mathsf R$ be $\FO$. We will show that the preconditions of \Cref{lem:reductions} are met, from which it will follow that $\NL \subseteq \AHAT^1_*$.
    First, graph connectivity is known to be $\NL$-complete under $\FO$ reductions \citep{immerman-1998-descriptive}.
    Second, \citet{merrill-2025-little} recently showed that mixed-masked log-depth transformers with cubic padding can recognize the graph connectivity problem $L$, i.e., $L \in \mAHAT^1_*$. From \Cref{prop:mask-conversion}, it follows that $L \in \AHAT^1_*$.
    Finally, it follows from \Cref{thm:tc0} that fixed-depth causally masked transformers can recognize languages in $\FOM$, and hence also languages in $\FO$. Such transformers can therefore compute $\FO$ reductions $f$ in the sense of \Cref{def:transformer-computes-reduction}, i.e., compute the function $r_f(w, i)$ defined as $f(w)_i$, the $i$-th bit of $f(w)$.
    Thus, \Cref{lem:reductions} applies.
\end{proof}

Since $\L$ reductions are in $\NL$, we can bootstrap this result to obtain the following stronger result. For this, we will leverage \CR{the notion of reductions and the} completeness of the problem of evaluating a given \CR{``wide''} log$^d$ depth circuit, formalized in \Cref{def:circuit-evaluation-problem} (\Cref{sec:circuit-evaluation-problem}).



\CR{
\begin{lemma} \label{lem:TCd-in-transformers}
    For $d \geq 0$, $\FOUniform\ \TC^d \subseteq \AHAT^d_*$.
\end{lemma}

\begin{proof}
    We will apply \Cref{lem:reductions} with the wide-$\TC^d$ circuit evaluation problem (\Cref{sec:circuit-evaluation-problem}) as $L$, $\FOUniform\ \TC^d$ as class $\mathsf C$, and $\L$ as class $\mathsf R$. We next argue that the preconditions of \Cref{lem:reductions} are met, which will finish the proof.
    First, \Cref{cor:circuit-evaluation-by-transformer} (\Cref{sec:solving-wide-TCd-with-transformers}) shows that log$^d$-depth looped transformers (without padding) can solve the wide-$\TC^d$ circuit evaluation problem, i.e., $L \in \AHAT^d_0$.
    Second, \cref{lem:wide-TCd-eval-complete-for-FOuniform-TCd} (\Cref{sec:solving-wide-TCd-with-circuits}) shows that $L$ is complete for $\FOUniform\ \TC^d$ under $\L$ reductions.
    Finally, \Cref{lem:nl-lb} implies that log$^d$-depth transformers for $d \geq 1$ can recognize any language in $\L$, and thus compute any $\L$ reduction in the sense of \Cref{def:transformer-computes-reduction}. Applying \Cref{lem:reductions}, we conclude that $\FOUniform\ \TC^d \subseteq \AHAT^d_*$.
\end{proof}
}

\CR{The proof heavily leverages the fact that wide-$\TC^d$ circuit evaluation $\TC^d$-complete, which we show in \Cref{sec:wide-TCd-hardness}.
To our knowledge, formalizing this $\TC^d$-complete problem (or, in fact, any natural $\TC^d$-complete problem) is a novel contribution.}
We next note the following extension of a known result about fixed-depth transformers. See~\Cref{sec:proofs} for a proof, \CR{which leverages the fact that recurrent composition of poly-size $\FOUniform$ circuit families remains $\FOUniform$ (\cref{lem:fo-iterated-composition}, \cref{sec:composition}):}

\begin{restatable}{lemma}{transformersInTCd}\label{lem:transformers-in-TCd}
    For $d \geq 1$, $\AHAT^d_* \subseteq \FOUniform\ \TC^d$.
\end{restatable}


\CR{Combining \cref{lem:transformers-in-TCd,lem:TCd-in-transformers}, we obtain an exact characterization of $\AHAT^d_*$ for $d \geq 1$:
\begin{theorem} \label{thm:tck}
    For any $d \geq 1$, $\AHAT^d_* = \FOUniform\ \TC^d$.
\end{theorem}

Taking the union over all $d$, we obtain $\bigcup_{d=0}^\infty \AHAT^d_*$ as the class of languages recognized by \emph{polylogarithmic-looped padded transformers}. \cref{thm:tc0,thm:tck} imply that this class is the same as $\bigcup_{d=0}^\infty \FOUniform\ \TC^d$, which in turn is the same as the class $\FOUniform\ \NC$, where $\NC$ is the class of all ``parallelizable'' languages---those recognized by polylogarithmic depth (and polynomial size) circuit families. We therefore have:

\begin{corollary}\label{cor:polylog-depth-recognize-NC}
    Polylog-looped poly-padded transformers recognize exactly $\FOUniform\ \NC$.
\end{corollary}

\CR{As before, we note that the arguments leading to \cref{thm:tck,cor:polylog-depth-recognize-NC} work for both logarithmic and polynomial precision (cf.~\cref{sec:datatype}) transformers. This shows that these two levels of precision lead to identical power for polylog-looped polynomially-padded transformers.}


}




\CR{
\section{Uniformity Collapse for Circuit Classes}
\label{sec:uniformity-collapse}

Our results on looped transformers involved substantial analysis of uniform polylogarithmic-depth circuit classes.
This analysis led us to prove a novel result about uniform circuit families, which both slightly strengthens our results about looped transfomers and may be of independent interest.

The result concerns the strength of different uniformity conditions for circuit families, characterizing conditions under which variants of circuit classes with different uniformity conditions \emph{collapse} to express the same class of languages (\cref{prop:uniformity-collapse} in \cref{appendix:uniformity-collapse}).
Applying this general result to (wide) $\TC^d$ circuits, we show that for $d \geq 1$, both $\NLUniformity$ and $\LUniformity$ collapse to the weaker notion of $\FOUniformity$ for both $\AC^d$ and $\TC^d$ circuits:

\begin{restatable}[Uniformity Collapse]{theorem}{uniformitycollapsetcd}\label{thm:uniformity-collapse-ACd-TCd}
    For any $d \geq 1$, the following equivalences hold:
    \begin{align*}
        \FOUniform\ \AC^d &= \LUniform\ \AC^d = \NLUniform\ \AC^d \\
        \FOUniform\ \TC^d &= \LUniform\ \TC^d = \NLUniform\ \TC^d .
    \end{align*}
\end{restatable}

The key idea, formalized in \cref{appendix:uniformity-collapse}, is to leverage the fact that $\NL$ and $\L$ themselves are in $\FOUniform\ \AC^d$ for $d \geq 1$, and therefore the $\NL$ or $\L$ machine that builds a circuit family can itself be simulated in $\FOUniform\ \AC^d$ as well. Thus all one needs to do is compose this ``circuit building'' circuit family with a ``circuit evaluation'' circuit family. To this end, we show that functions computed by $\FOUniform$ circuit families are closed under fixed compositions (\cref{prop:fixed-composition}).

\Cref{thm:uniformity-collapse-ACd-TCd} implies that, for $d \geq 1$, $\AHAT^d_*$ recognizes not just $\FOUniform\ \TC^d$ but also $\LUniform\ \TC^d$ because these classes are the same.
}

\section{Conclusion}

Our results in this work give a precise theoretical understanding of how padding and looping---two ways to dynamically expand the computational resources of a transformer at inference time---increase the expressive power of transformers.
Padding expands the circuit \emph{width} of transformers, allowing them to resolve logical formulas over more variables.
As a consequence of this, polynomially padded transformers can recognize \emph{exactly} $\TC^0$, which was previously known only as an upper bound \citep{merrill-sabharwal-2023-parallelism,pfau-2024-think}.
In contrast, looping increases the \emph{depth} of transformers.
Applying looping on top of padding, we extended our result to show that log$^d$-depth, padded transformers recognize exactly $\TC^d$.
This means that transformers with polynomial padding and polylogarithmic looping converge to recognizing $\NC$, the largest class of problems that can be solved with parallel computation.
In contrast, transformers with CoT have greater expressive power under standard complexity conjectures \citep{merrill2023cot}, but 
suffer from slow sequential decoding.
Thus, while looping and padding are not as powerful as CoT, our results suggest they are quite effective ways to expand transformers' expressive power while preserving parallelism.

Several interesting open questions remain from this work.
On the theoretical side, it would be valuable to develop a more finegrained characterization of $\AHAT^d_k$, i.e., looped transformers where the padding is at most $\O(n^k)$ for some \emph{fixed} $k$.
On the empirical side, while looped and padded transformers have both been explored already to an extent, it would be interesting to see whether these approaches could be successfully integrated to improve the performance of transformers on hard reasoning tasks.
Further developing these approaches could ultimately lead to inference-time compute methods based on padding and looping that increase transformers' expressive power without sacrificing parallelism, providing a more efficient alternative to CoT for solving moderately parallelizable problems.

\CR{\section*{Limitations}

\noindent
\textbf{Practical Concerns for Looped Transformers with Padding.}
In this paper, we showed looped transformers with padding are quite \emph{expressive}. However, the degree to which transformers can learn to use looped layers is an important open question.
In particular, practical details might be important here, such as an appropriate parameterization where features can be learned in later layers \citep{dey2025completep}.
Another practical caveat is that, while padding is efficiently parallelizable, it will extend the context length, incurring more memory overhead for the forward pass.

\noindent
\textbf{Transformer Model.}
Here we have assumed AHATs here rather than softmax-attention transformers (SMATs).
For any fixed maximum context length, it is possible scale the temperature of SMAT heads to arbitrarily approximate AHAT heads, but for unbounded context lengths, SMATs may not be able to simulate AHATs. We thus view the AHAT as mild simplification of the SMAT that abstracts away issues with soft attention for simulating hard attention over very long contexts.
It would also be interesting to better understand the necessity of the masked pre-norm assumption.
}

\section*{Acknowledgments}
We appreciate discussions with Selim Jerad, Andy Yang, Michael Cadhilac, and attendees of the Formal Languages and Neural Networks (FLaNN) seminar.
This project was supported by the National Science Foundation (NSF) through award 1922658 and WM's NSF Graduate Research Fellowship. WM was also supported by a Two Sigma PhD Fellowship and the Allen Institute for AI.

\bibliographystyle{abbrvnat}
\bibliography{references}

@inproceedings{
    pfau-2024-think,
    title={Let{\textquoteright}s Think Dot by Dot: Hidden computation in transformer language models},
    author={Jacob Pfau and William Merrill and Samuel R. Bowman},
    booktitle={COLM},
    year={2024},
    url={https://openreview.net/forum?id=NikbrdtYvG}
}

@inproceedings{
    merrill-2025-little,
    title={A Little Depth Goes a Long Way: The Expressive Power of Log-Depth Transformers},
    author={William Merrill and Ashish Sabharwal},
    booktitle={NeurIPS},
    year={2025},
    url={https://openreview.net/forum?id=5pHfYe10iX}
}

@article{merrill-sabharwal-2023-parallelism,
    title = "The Parallelism Tradeoff: Limitations of Log-Precision Transformers",
    author = "Merrill, William  and
      Sabharwal, Ashish",
    journal = "TACL",
    volume = "11",
    year = "2023",
    url = "https://aclanthology.org/2023.tacl-1.31",
    doi = "10.1162/tacl_a_00562",
    pages = "531--545",
}

@article{barrington-1990-uniformity,
    title = {On uniformity within {NC1}},
    journal = {Journal of Computer and System Sciences},
    volume = {41},
    number = {3},
    pages = {274-306},
    year = {1990},
    issn = {0022-0000},
    doi = {https://doi.org/10.1016/0022-0000(90)90022-D},
    url = {https://www.sciencedirect.com/science/article/pii/002200009090022D},
    author = {David A. {Mix Barrington} and Neil Immerman and Howard Straubing},
}

@inproceedings{merrill2023logic,
    title={A Logic for Expressing Log-Precision Transformers},
    author={William Merrill and Ashish Sabharwal},
    booktitle={NeurIPS},
    year={2023},
    url={https://openreview.net/forum?id=uR8TtWCIsr}
}

@article{chiang2025transformers,
    title={Transformers in Uniform {TC}\${\textasciicircum}0\$},
    author={David Chiang},
    journal={TMLR},
    issn={2835-8856},
    year={2025},
    url={https://openreview.net/forum?id=ZA7D4nQuQF},
    note={}
}

@inproceedings{wei2022chain,
    title={Chain of Thought Prompting Elicits Reasoning in Large Language Models},
    author={Jason Wei and Xuezhi Wang and Dale Schuurmans and Maarten Bosma and brian ichter and Fei Xia and Ed H. Chi and Quoc V Le and Denny Zhou},
    booktitle={NeurIPS},
    year={2022}
}

@inproceedings{merrill2023cot,
    title={The Expressive Power of Transformers with Chain of Thought},
    author={William Merrill and Ashish Sabharwal},
    booktitle={NeurIPS 2023 Workshop on Mathematics of Modern Machine Learning},
    year={2023},
    url={https://openreview.net/forum?id=CDmerQ37Zs}
}

@article{hao-etal-2022-formal,
    title = "Formal Language Recognition by Hard Attention Transformers: Perspectives from Circuit Complexity",
    author = "Hao, Yiding  and
      Angluin, Dana  and
      Frank, Robert",
    journal = "TACL",
    volume = "10",
    year = "2022",
    url = "https://aclanthology.org/2022.tacl-1.46/",
    doi = "10.1162/tacl_a_00490",
    pages = "800--810",
}

@inproceedings{yang2024masked,
    title={Masked Hard-Attention Transformers Recognize Exactly the Star-Free Languages},
    author={Andy Yang and David Chiang and Dana Angluin},
    booktitle={NeurIPS},
    year={2024},
    url={https://openreview.net/forum?id=FBMsBdH0yz}
}

@article{merrill-etal-2022-saturated,
    title = "Saturated Transformers are Constant-Depth Threshold Circuits",
    author = "Merrill, William  and
      Sabharwal, Ashish  and
      Smith, Noah A.",
    journal = "TACL",
    volume = "10",
    year = "2022",
    url = "https://aclanthology.org/2022.tacl-1.49/",
    doi = "10.1162/tacl_a_00493",
    pages = "843--856",
}

@misc{ba2016layernormalization,
    title={Layer Normalization}, 
    author={Jimmy Lei Ba and Jamie Ryan Kiros and Geoffrey E. Hinton},
    year={2016},
    eprint={1607.06450},
    archivePrefix={arXiv},
    primaryClass={stat.ML},
    url={https://arxiv.org/abs/1607.06450}, 
}

@InProceedings{xiong-2020-layer,
    title = 	 {On Layer Normalization in the Transformer Architecture},
    author =       {Xiong, Ruibin and Yang, Yunchang and He, Di and Zheng, Kai and Zheng, Shuxin and Xing, Chen and Zhang, Huishuai and Lan, Yanyan and Wang, Liwei and Liu, Tieyan},
    booktitle = 	 {ICML},
    year = 	 {2020},
    url = 	 {https://proceedings.mlr.press/v119/xiong20b.html},
}

@inproceedings{zhang2019rms,
    author = {Zhang, Biao and Sennrich, Rico},
    title = {Root mean square layer normalization},
    year = {2019},
    booktitle = {NeurIPS}
}

@inproceedings{li2024chain,
    title={Chain of Thought Empowers Transformers to Solve Inherently Serial Problems},
    author={Zhiyuan Li and Hong Liu and Denny Zhou and Tengyu Ma},
    booktitle={ICLR},
    year={2024},
    url={https://openreview.net/forum?id=3EWTEy9MTM}
}

@article{strobl-etal-2024-formal,
    title = "What Formal Languages Can Transformers Express? {A} Survey",
    author = "Strobl, Lena  and
      Merrill, William  and
      Weiss, Gail  and
      Chiang, David  and
      Angluin, Dana",
    journal = "TACL",
    volume = "12",
    year = "2024",
    url = "https://aclanthology.org/2024.tacl-1.30/",
    doi = "10.1162/tacl_a_00663",
    pages = "543--561",
}

@book{arora-2009-computational,
    author = {Arora, Sanjeev and Barak, Boaz},
    title = {Computational Complexity: A Modern Approach},
    year = {2009},
    isbn = {0521424267},
    publisher = {Cambridge University Press},
    address = {USA},
    edition = {1st},
}

@book{immerman-1998-descriptive,
    author = {Neil Immerman},
    editor = {},
    publisher = {Springer Verlag},
    title = {Descriptive Complexity},
    year = {1998}
}

@misc{nye2022show,
    title={Show Your Work: Scratchpads for Intermediate Computation with Language Models},
    author={Maxwell Nye and Anders Johan Andreassen and Guy Gur-Ari and Henryk Michalewski and Jacob Austin and David Bieber and David Dohan and Aitor Lewkowycz and Maarten Bosma and David Luan and Charles Sutton and Augustus Odena},
    year={2022},
    url={https://openreview.net/forum?id=iedYJm92o0a}
}

@inproceedings{goyal2024think,
    title={Think before you speak: Training Language Models With Pause Tokens},
    author={Sachin Goyal and Ziwei Ji and Ankit Singh Rawat and Aditya Krishna Menon and Sanjiv Kumar and Vaishnavh Nagarajan},
    booktitle={ICLR},
    year={2024},
    url={https://openreview.net/forum?id=ph04CRkPdC}
}

@inproceedings{chiang-2023-tighter,
    author = {Chiang, David and Cholak, Peter and Pillay, Anand},
    title = {Tighter bounds on the expressivity of transformer encoders},
    year = {2023},
    booktitle = {ICML}
}

@inproceedings{yang2024counting,
    title={Counting Like Transformers: Compiling Temporal Counting Logic Into Softmax Transformers},
    author={Andy Yang and David Chiang},
    booktitle={COLM},
    year={2024},
    url={https://openreview.net/forum?id=FmhPg4UJ9K}
}

@inproceedings{dehghani2018universal,
    title={Universal Transformers},
    author={Mostafa Dehghani and Stephan Gouws and Oriol Vinyals and Jakob Uszkoreit and Lukasz Kaiser},
    booktitle={ICLR},
    year={2019},
    url={https://openreview.net/forum?id=HyzdRiR9Y7},
}

@inproceedings{giannou-2023-looped,
    author = {Giannou, Angeliki and Rajput, Shashank and Sohn, Jy-yong and Lee, Kangwook and Lee, Jason D. and Papailiopoulos, Dimitris},
    title = {Looped transformers as programmable computers},
    year = {2023},
    booktitle = {ICML}
}

@inproceedings{lange-2004-results,
    author = {Lange, Klaus-Jorn},
    title = {Some Results on Majority Quantifiers over Words},
    year = {2004},
    isbn = {0769521207},
    publisher = {IEEE Computer Society},
    booktitle = {Conference on Computational Complexity},
    pages = {123–129}
}

@misc{dey2025completep,
    title={Don't be lazy: {CompleteP} enables compute-efficient deep transformers}, 
    author={Nolan Dey and Bin Claire Zhang and Lorenzo Noci and Mufan Li and Blake Bordelon and Shane Bergsma and Cengiz Pehlevan and Boris Hanin and Joel Hestness},
    year={2025},
    eprint={2505.01618},
    archivePrefix={arXiv},
    primaryClass={cs.LG},
    url={https://arxiv.org/abs/2505.01618}, 
}

@book{vollmer1999introduction,
  title={Introduction to circuit complexity: a uniform approach},
  author={Vollmer, Heribert},
  year={1999},
  publisher={Springer Science \& Business Media}
}

\clearpage
\appendix
\CR{\section{Datatype Assumptions} \label{sec:datatype}

We adapt the $p$-precise datatype model from \citet{merrill-2025-little}.
We encode scalars as strings in $\{0, 1\}^p$, where $p$ can be a function of $n$.
If $p$ depends on $n$, activations, but not models parameters, can depend on $n$.
A \emph{datatype} $\mathbb D_p$ assigns a numerical semantics for each string in $\{0, 1\}^p$.
For $x \in \mathbb R$, let $[x]_{\mathbb D_p}$ be $x$ rounded into $\mathbb D_p$, i.e., the bitstring whose numerical value in $\mathbb D_p$ is closest to $x$ (breaking ties in favor of the higher value).
We define our datatype $\mathbb D_p$ to satisfy the following:

\begin{definition}[$p$-Precise Operations] \label{def:precise}
    Let $f: \mathbb R^k \to \mathbb R$ be
    an operation
    with $p$-precision realization $\tilde f: \mathbb D_p^k \to \mathbb D_p$.
    We say $\tilde f$ is $p$-precise if, for any $x_1, \ldots, x_k \in \mathbb R$ exactly representable in $\mathbb D_p$,
    \begin{equation*}
        [f(x_1, \ldots, x_k)]_{\mathbb D_p} = \tilde f([x_1]_{\mathbb D_p}, \ldots, [x_k]_{\mathbb D_p}) .
    \end{equation*}
\end{definition}

To apply \Cref{def:precise}, we view the summation in attention heads as an $n$-ary operation.
We also view layer-norm as a single operation from $\mathbb R^m \to \mathbb R^m$.
Lastly, we assume that these operations in the computation graph are defined in the standard way and are thus computable in $\TC^0$ \citep{merrill-sabharwal-2023-parallelism,merrill2023logic}.

We consider two natural instantiations of $\mathbb D_p$ in the main text: \emph{log precision}, with $p = c \log n$ \citep{merrill-sabharwal-2023-parallelism}, and \emph{polynomial precision}, with $p = n^c$ \citep{chiang2025transformers}, for some fixed $c > 0$. All our results go through with either datatype, showing their equivalence for polylogarithmically-looped polynomially-padded transformers.
}

\CR{
\section{Uniformity of Circuit Classes} \label{appendix:uniformity}

To make our arguments rigorous for highly uniform (i.e., $\FOUniform$) circuit families, it will be necessary to work with a detailed definition of uniformity. For this, we start with a standard formal definition of the \emph{connection language} describing the gates and wires of a circuit family. Here we allow each circuit to have more than one (numbered) output gates.

\begin{definition}[Connection Language; cf.~\citealp{vollmer1999introduction}]
    Define the connection language for a circuit family $\mathcal C = \{C_n\}_{n=0}^\infty$ as $L_{\mathcal C} = L_\textrm{gate} \cup L_\textrm{wire}$ where
    \begin{align*}
        L_\textrm{gate} &= \{ a^n i g k \ell \mid \textrm{gate $i$ of $C_n$ is type $g$, is $k$-th input gate if $k \geq 0$, and is $\ell$-th output gate if $\ell \geq 0$} \} \\
        L_\textrm{wire} &= \{ a^n i j \mid \textrm{$C_n$ has a wire $i \to j$} \}
    \end{align*}
    and $i, j \in \mathbb{Z}_{\geq 0}$, $g$ is a gate type, and $k, \ell \in \mathbb{Z}_{\geq -1}$. Here $i, j, k, \ell$ are represented using exactly $c \log n$ bits for some fixed $c \in \mathbb{Z}_{\geq 1}$.
\end{definition}

We will assume that gates in all circuits are numbered so that there is a block of input gates, followed by a block of intermediate gates, followed by a block of output gates. Thus, there exist specific gate thresholds that denote the boundaries for input and output gates.

\begin{definition}[Generalized Uniformity]
    A family of circuits $\mathcal C = \{C_n\}_{n=0}^\infty$ is $\AUniform$ if $L_{\mathcal C} \in \A$.
\end{definition}

When $\A = \L$, this generalized uniformity notion is equivalent to the standard definition of uniformity in terms of being able to serialize $C_n$ as a function of $1^n$:

\begin{proposition}\label{prop:equivalent-L-uniform-defn}
    $\mathcal C = \{C_n\}_{n=0}^\infty$ is $\LUniform$ if and only if $1^n \mapsto \langle C_n \rangle$ is computable in log space.
\end{proposition}

\begin{proof}
    To use the connection language notion of uniformity to serialize the circuit with log space, we simply maintain a counter of our current gate position and edge and call an $\L$ oracle for the connection language to print each gate and wire. In the other direction, we modify the routine that serializes the circuit to only output gate $i$ or edge $i, j$, and we use this to recognize $L_{\mathcal C}$.
\end{proof}

\subsection{Uniformity Composition Closures} \label{sec:composition}


Our later results about uniform circuit families will rely on the property (formalized below) that for strong-enough classes $\A$, any fixed composition of $\AUniform$ circuit families is also $\AUniform$. Note that any fixed composition of functions can be decomposed into a fixed number of \emph{serial} and \emph{parallel} compositions of two functions, defined as follows:

\begin{definition}[Function Composition]
    For functions $f, g \in \{0, 1\}^* \to \{0, 1\}^*$, their \emph{serial composition} is the function $(g \circ f)(w) = g(f(w))$ and their parallel composition is the function $\langle f, g \rangle (w) = f(w) \cdot g(w)$.
\end{definition}

For $\LUniform$ functions, serial and parallel composition clearly preserves $\LUniformity$, as we can maintain a counter to reroute inputs and outputs between subcircuits appropriately. The same applies to any class $\A$ containing $\L$.
For $\FO$ uniform circuit classes, closure under composition requires a bit more work to justify, as shown in the following two lemmas.

\begin{lemma}\label{lem:fo-parallel-composition}
    If functions $f$ and $g$ have polynomial-size $\FOUniform$ circuit families, then so does their parallel composition $\langle f, g \rangle$.
\end{lemma}

\begin{proof}
    Let $\mathcal C^f = \{C^f_n\}_{n=0}^{\infty}$ be an $\FOUniform$ circuit family for $f$, with connection language $L_f = L_{\mathcal C_f} \in \FO$. Similarly, define $\mathcal C^g, C^g_n,$ and $L_g$ for $g$.
    
    We will decide the connection language $L_{\langle f, g \rangle}$ by querying $L_f$ and $L_g$ to build $C^f_n$ and $C^g_n$ in parallel.
    Without loss of generality, we assume all gate indices in the input word for a gate query fit into one of the following three cases: (1) at most the number of gates $s_f$ in the circuit $C^f_n$, (2) larger than $s_f$ but at most $s_f + s_g$ where $s_g$ is the number of gates in $C^g_n$, or (3) larger than $s_f + s_g$ but at most $s_f + s_g + o_f$, where $o_f$ is the number of output gates of $C^f_n$.
    If all gates are at most $s_f$, we simply query $L_f$.
    If all gates are larger than $s_f$ but at most $s_f + s_g$, we subtract $s_f$ from each gate index and query $L_g$.
    Finally, if we are querying new gates from $s_f + s_g < i \leq s_f + s_g + o_f$, we design the connection language to represent identity gates.
    Importantly, addition and inequality checks can be performed in $\FO$, and the connection languages $L_f$ and $L_g$ can be queried in $\FO$ by construction.

    The logic for edge queries in the connection language modifies gate indices according to the same logic as the gate queries with a few exceptions. First, for an edge $(i, j)$, if $i$ is an input for $C^f_n$ (and $C^g_n$) and $i = j - s_f$, we return $1$. This ensures that $C^f_n$ and $C^g_n$ both receive the same input.
    Second, if $i$ is an output for $C^f_n$ and $s_f + s_g < j \leq s_f + s_g + o_f$, we also return 1.
    This copies over the outputs of $C^f_n$ so that they appear at the very end of the combined circuit.

    Thus, this connection language constructs a circuit family that computes $f$ and $g$ in parallel and returns their outputs $\langle g(w), f(w) \rangle$. Without loss of generality, their order can be easily permuted to obtain the output $\langle f(w), g(w) \rangle$ of parallel composition. 
    It is clear from the construction that the size of the resulting circuit family is linear in the sizes of $\mathcal C^f$ and $\mathcal C^g$, and thus polynomial in $n$.
\end{proof}

\begin{lemma}\label{lem:fo-serial-composition}
    If functions $f$ and $g$ have polynomial-size $\FOUniform$ circuit families, then so does their serial composition $g \circ f$.
\end{lemma}

\begin{proof}
    Let $\mathcal C^f = \{C^f_n\}_{n=0}^{\infty}$ be an $\FOUniform$ circuit family for $f$, with connection language $L_f = L_{\mathcal C_f} \in \FO$. Similarly, define $\mathcal C^g, C^g_n,$ and $L_g$ for $g$.
    
    We will decide the connection language $L_{g \circ f}$ by querying $L_f$ to construct $C^f_n$, then modifying the inputs to $L_g$ to build $C^g_m$ on the output of $f$, which has size $m$.
    The cases are the same as for parallel composition, except for a few changes.
    First, we build $C^g_m$ instead of $C^g_n$.
    Second, when querying input gates for $C^g_m$, we add a condition that routes from the \emph{outputs} of $C^f_n$ rather than its inputs.
    Finally, we do not construct additional gates to copy over the outputs from $C^f_n$.
    Thus, this connection language builds a circuit family that computes $g$ on the output of $f$.
    Similar to the parallel composition case, it is clear from the construction that the size of the resulting circuit family is linear in the sizes of $\mathcal C^f$ and $\mathcal C^g$, and thus polynomial in $n$.
\end{proof}

Combining \cref{lem:fo-parallel-composition,lem:fo-serial-composition}
along with the fact that any fixed composition of functions can be decomposed into a fixed number of serial or parallel compositions of two functions at a time, we obtain that fixed function composition preserves $\FOUniformity$ of circuit families:

\begin{proposition}[Composition Preserves $\FO$-Uniformity]
\label{prop:fixed-composition}
    Any fixed composition of polynomial-size $\FOUniform$ circuit families is also polynomial-size and $\FOUniform$.
\end{proposition}
}

\CR{
\begin{lemma}[Recurrent Composition] \label{lem:fo-iterated-composition}
    Let $m(n)$, $d(n)$, and $r(n)$ be functions at most polynomial in $n$, with $m(n)$ and $r(n)$ definable as variables in $\FO$ given $a^n$.
    Let $f: \{0, 1\}^{m(n)} \to \{0, 1\}^{m(n)}$ have a polynomial-size $\FO$-uniform circuit family with depth $d(n)$.
    Then the function $f^{r(n)}$ (i.e., $f$ called recurrently on itself $r(n)$ times) has an $\FO$-uniform circuit family of polynomial size and depth $d(n)\; r(n)$.
\end{lemma}

\begin{proof}
    Let $\mathcal C^f = \{ C^f_n \}_{n=0}^\infty$ be a polynomial size $\FO$-uniform circuit family for $f$, with connection language $L_f$.
    Without loss of generality, we assume the number of gates in $C^f_n$ is $2^{s(n)}$ for some integer $s(n)$ (this can always be achieved by padding $C^f_n$ if necessary, in a way that increases size by at most a factor of 2). Since the circuit family is of polynomial size, $s(n) = \O(\log n)$ and can be computed in $\FO$ as the smallest $j$ s.t.\ $\bit(s_f, j) = 1$, where $s_f$ is the number of gates in $C^f_n$.
    We will construct an $\FO$-uniform circuit family $\mathcal C = \{C_n\}_{n=0}^\infty$ where $C_n$ consists of $f(n)$ iterations of $C^f_n$. By construction, $C_n$ has size at most polynomial and depth $d(n)\; r(n)$.
    
    The remainder of the proof will justify that $\mathcal C$ is $\FO$-uniform by defining its connection language $L_{\mathcal C}$ in terms of $L_f$, starting with gate queries. Given a gate query $w$ with index $i$, we first compute $r(n)$, which is possible by construction.
    We then compute $S = r(n) \cdot 2^{s(n)}$ by left shifting $r(n)$ by $s(n)$, and, if $i \geq S$, we reject $w$.
    Otherwise, we compute $i' = i \mod 2^{s(n)} = \bit(i, s(n))$,
    which can be computed in $\FO$ by reading the $s(n)$-th bit from the start of $i$ in $w$.
    We then compute $u$ as a new query where $i$ is replaced by $i'$. We then query whether $u \in L_f$. This ensures that $C_n$ repeats all the gates in $C^f_n$ stacked in $r(n)$ blocks.

    Given an edge query $w$ between $(i, j)$, we first compute $i' = i \bmod 2^{s(n)}$ and $j' = j \mod 2^{s(n)}$ as above.
    Additionally, let $q_i = \floor{i / 2^{s(n)}}$ and $q_j = \floor{j / 2^{s(n)}}$. We compute $q_i$ and $q_j$ in $\FO$ by right-shifting by $s(n)$. If $q_i = q_j$, we follow similar logic to the gate case, constructing $u$ by replacing $i$ with $i'$ and $j$ with $j'$. We then query whether $u \in L_f$. This has the effect of constructing all edges within a block of $C_n$ analogously to those in $C^f_n$.
    Additionally, if $q_i + 1 = q_j$, $i' \geq s_f - m(n)$, and $j' \leq m(n)$, we return a 1 for this edge; recall that $m(n)$ is the number of inputs as well as the number of outputs of $f$.
    This has the effect of routing the output of each block as the input for the next block.
    Thus, $\mathcal C$ computes $f^{r(n)}$ and is $\FO$-uniform.
\end{proof}
}

\section{Omitted Proofs}\label{sec:proofs}

\maskconversion*

\begin{proof}
    \CR{As noted in the proof of \cref{lem:mask-conversion}, the $1/i$ position encoding used in our unmasked transformers can be computed by causally masked transformers, and hence also by mixed-masked transformers.
    Thus, unmasked padded looped transformers with position encoding $1/i$ constitute a special case of mixed-masked padded looped transformers, and can thus be trivially simulated by the latter.} We will next describe how to \CR{leverage \cref{lem:mask-conversion} to} convert an unmasked padded \CR{looped} transformer to a masked one, on a per-head basis. The construction will leave the computation of masked heads unchanged, making the approach suitable for converting both unmasked and mixed-masked padded transformers to masked ones.

    Let $E$ be an unmasked \CR{looped} encoder in $\uAHAT^d_k$. Then $E$ has depth $\ell = O(\log^d n)$ and operates over $n + O(n^k)$ tokens (including original input and padding tokens). By \Cref{lem:mask-conversion}, $E$ can be simulated by a causally masked decoder $D$ with depth $\ell$ and with $\ell \cdot (n + O(n^k))$ new padding tokens (after the original $\O(n^k)$ padding tokens used by $E$). Thus, the total number of padding tokens $D$ uses is $p = O(n^k) + \ell \cdot (n + O(n^k))$. Since $\ell = O(\log^d n)$, this simplifies to $p = O(\log^d(n) \cdot n^{k'})$ where $k' = \max\{k, 1\}$. This is $\O(n^{k'})$ when $d = 0$, and $\O(n^{k' + 1})$ when $d \geq 1$. This finishes the proof of the first two parts. For the third part, observe that by definition, $\uAHAT^d_* = \bigcup_{k=0}^\infty \uAHAT^d_k$. From the above, this in turn is contained in $\bigcup_{k=0}^\infty \AHAT^d_{k' + 1} \subseteq \AHAT^d_*$, as desired.

    For the mixed-masking case, the construction in \Cref{lem:mask-conversion} will preserve the original computation of causally masked heads if we add an additional term with large negative weight $C$ (fixed w.r.t.~$n$) that is activated if the index within the block is greater than $b_i$.
    We set $C$ large enough to dominate all other terms in the inner product computation at each token.
    As a result, the head is constrained to only attend to previous tokens within the block.
    Applying this modified construction on a case-by-case basis per head, we can take this proposition to apply equally to transformers with mixed-masking.
\end{proof}

\CR{The above simulation of unmasked padded looped transformers (with position encoding $1/i$) with mixed-mask transformers can, in fact, be extended even to the case where the unmasked transformer uses $i/n$ positional encoding:

\begin{lemma} \label{lem:masked-pos}
    There exists a mixed-mask sublayer with one masked head and one unmasked head that computes $\phi(i/n)$.
\end{lemma}

\begin{proof}
    We use two attention heads, one causally masked and one unmasked, both of which attend uniformly with value $1$ only for the beginning-of-sequence symbol. The first head thus computes $1/i$, while the second head computes $1/n$. We can then combine these values to compute $\phi(1/n, 1/i) = \phi(i/n)$.
\end{proof}

\ReductionsForTransformers*

\begin{proof}
    Consider any language $L' \in \mathsf C$. By the assumed completeness of $L$, there exists an $\mathsf R$ reduction $f$ that maps inputs of $L'$ into inputs of $L$ such that $w \in L'$ if and only if $f(w) \in L$. From the last precondition of the theorem, there exists a causally masked log$^d$-depth padded transformer $T_f$ that recognizes the corresponding reduction language $R_f$ or, equivalently, computes the reduction function $r_f$ (cf.~\Cref{def:transformer-computes-reduction}). We will assume the latter, i.e., that $T_f$ computes $r_f$, though the construction can also be made to work if $T_f$ checks membership in $R_f$. Additionally, we also have from a precondition that there exists a causally masked log$^d$-depth padded transformer $T_L$ that recognizes $L$.

    The idea is to ``stack'' $T_L$ on top of $T_f$ to obtain a log$^d$-depth padded transformer $T$ that recognizes $L'$. Intuitively, given an input $w$, the first set of layers of $T$ will compute $f(w)$ tokenwise, by computing $r(w, i) = f(w)_i$ for every $i$ in parallel. To this end, we will essentially make $\abs{f(w)}$ copies of the padding tokens needed by $T_f$ and perform the computation of $T_f$ independently for each $f(w)_i$. The second and final set of layers of $T$ will then check whether the string $f(w)$ produced by the first set of layers is in $L$, which will hold if and only if $w \in L'$. We next make this idea more concrete.

    \paragraph{}\underline{Compute Reduction}.
    Suppose $T_f \in \AHAT^d_k$ and let $n = \abs{w}$.
    Then, on input $(w,i)$, where $i$ is represented in binary, $T_f$ uses $\O(n^k)$ padding tokens to compute $f(w)_i$, which we upper bound by $n^{k+1}$ for sufficiently large $n$ (for small $n$, we assume $T_f$ instead uses a fixed lookup table (and no padding).
    There is a uniform $\TC^0$ circuit family that computes $\bit(i, j)$ from input $\phi(i), \phi(j)$, so, by \Cref{thm:tc0}, there exists a transformer $T_{\bit}$ that computes $\bit$ with $n^b$ padding tokens, for some $b$ and sufficiently large $n$ (again, for smaller $n$, we assume $T_{\bit}$ uses a fixed lookup table).
    Since $\abs{f(w)}$ is bounded by some polynomial $n^c$, we know that any $i \leq \abs{f(w)}$ takes at most $c \log n$ bits to specify, which we upper bound by $c' = n$ (since it's unclear how a transformer would compute $\log n$).

    We will construct $n^c$ blocks of padding tokens, each of size $B = c' n^b + n^{k+1}$. Block $i$ will first compute the binary expansion of $i$ in its first $c' n^b$ padding tokens using $T_{\bit}$, and then use the remaining $n^{k+1}$ padding tokens to consume $(w,i)$ using $T_f$ and return $r_f(w,i)$.
    At each token $t$ in block $i$, we compute $\phi(i)$ as $\phi(\floor{t/B})$, using the fact that we can compute integer division with a fixed block of transformer layers \citep{merrill-2025-little} and having computed $B$ as a function of $n$ in earlier layers (since it is in $\TC^0$).
    For $t' \leq c'$, token $t = t' n^b$ computes $\phi(t')$ via division similarly to $\phi(i)$ and then $\bit(i, t')$. This recovers the binary representation of $i$ stored across tokens $t = t' n^b$ in block $i$, for $1 \leq t' \leq c'$.
    Finally, we apply a slightly modified $T_f$ over the block.
    Specifically, we add a new term to each attention head so that it only attends over the input tokens and some tokens within the current block: those satisfying $t = t' n^b$ for some $1 \leq t' \leq c'$ or $t > c' n^b$: since $c' = n$, these predicates are simple to check.
    As a result, block $i$ simulates $T_f$ over $w$ with $n^k$ padding.
    Thus, the final token of block $i$ computes $r_f(w, i) = f(w)_i$.


    \paragraph{}\underline{Solve Complete Problem}. 
    We will use additional layers to check whether $f(w) \in L'$.
    We are given that there exists a transformer $T_L$ that, on input $w'$, checks whether $w' \in L'$.
    For each attention head in $T_L$, we add a new term to the attention score that is very negative if that token is not the final token in some attention block.
    Thus, each head in $T_L$ will only attend over tokens that are final in some block.
    We also modify $T_L$ so that it uses $\phi(i)$ in place of the position embedding for token $t$.
    Thus, $T_L$ computes whether $f(w) \in L'$, which is equivalent to recognizing whether $w \in L$.
\end{proof}
}

\transformersInTCd*

\begin{proof}
    Let $L$ be a language in $\AHAT^d_*$ and let $T$ be a looped, padded AHAT transformer that, when unrolled $c \log^d n$ times for large enough $n$,\footnote{As before, for small $n$, we assume $T$ uses a lookup table.} recognizes whether $w \in L$ for any input $w$ with $\abs{w} = n$. Let $\langle A, B, C \rangle$ be the partition of layers of $T$ where $A$ is the set of initial layers, $B$ is the block that's repeated $c \log^d n$ times on inputs of length $n$, and $C$ is the set of final layers. Each of these itself is a fixed-depth padded transformer; let's call these $T_A$, $T_B$, and $T_C$. By prior results \citep{merrill-sabharwal-2023-parallelism,merrill2023logic,chiang2025transformers}, there are \CR{$\FOUniform$ $\TC^0$} circuit families $\{C^A_n\}_{n=0}^\infty, \{C^B_n\}_{n=0}^\infty,$ and $\{C^B_n\}_{n=0}^\infty$ that simulate transformers $T_A, T_B,$ and $T_C$, respectively.
    \CR{
    Let $T_R$ be $T_B$ iterated $\log n$ times, and Let $T^d_R$ be $T_B$ iterated $\log^d n$ times.
    We justify that $r(n) = \ceil{\log n}$ is definable as a variable in $\FO$ given $a^n$: compute the index of the last $a$ and then find the greatest bit index that is 1.
    We now invoke \Cref{lem:fo-iterated-composition} to show that $T_R$ has an $\FOUniform\ \TC^d$ circuit family; we can repeat this process $d$ times to get an $\FOUniform\ \TC^d$ circuit family for $T^d_R$.
    By \Cref{lem:fo-serial-composition}, there is a also an $\FOUniform\ \TC^d$ circuit family that computes the serial composition of $T_A, T_R, \ldots, T_R,$ and $T_C$, where $\ldots$ accounts for a fixed repetition of $T_R$ to account for a constant $c$ on the depth $c \ceil{\log n}^d$.
    Thus, we conclude that $L \in \FOUniform \TC^d$.}
\end{proof}

\section{Wide-\texorpdfstring{$\TC^d$}{TCd} Circuit Evaluation Problem}
\label{sec:circuit-evaluation-problem}

To formalize the circuit evaluation problem, we will use the following serialized format for representing a circuit. This format is a simplification of the one used by \citet{merrill-sabharwal-2023-parallelism}, with two main differences: (a) instead of a single threshold gate, we use $\andgate, \orgate, \notgate,$ and $\majoritygate$ (the majority gate), which will simplify the description of the construction; and (b) we do not require the gates in the serialization to be sorted in any particular order, which makes it easier (and perhaps even possible) to have a log-space reduction from any $\LUniform$ $\TC^d$ language to the circuit evaluation problem in this specific serialization format. The syntax of this circuit format is governed by the following grammar:
    \begin{small}
    \begin{align*}
    \circuit \ & \to \ \gate^* \\
    \gate \ & \to \ \xx \; \mid \; \operator \; \argument^* \\
    \operator \ & \to \ \andgate \; \mid \; \orgate \; \mid \; \notgate \; \mid \; \majoritygate \\
    \argument \ & \to \ \texttt{\&}\texttt{1}^*
    \end{align*}
    \end{small}
Semantically, we take the $k$-th $\xx$ gate to return the $k$-th input bit.
Other gates retrieve the values of the gates referred to by their argument pointers and apply the associated logical function.
We take the final gate in the serialization to be the output gate of the circuit.
Note that not all strings in this grammar represent a well-formed circuit, but any valid circuit can be serialized in this format.

As an example, the threshold circuit $\text{Majority}(x_1, x_2 \lor x_3, \neg x_3)$ stating that at least two of $x_1$, $x_2 \lor x_3$, and $\neg x_3$ should be true, would be represented as follows:
    $$
    \underbrace{\xx \; \xx \; \xx}_{\text{input}} \;
    \underbrace{\majoritygate \; \texttt{\&1} \; \texttt{\&11111} \; \texttt{\&111111}}_{\text{Majority gate}} \;
    \underbrace{\orgate \; \texttt{\&11} \; \texttt{\&111}}_{\text{Or gate}} \;
    \underbrace{\notgate \; \texttt{\&111}}_{\text{Not gate}}
    $$

Note that the non-input gates need not be serialized in this particular order. The following is also an equally valid serialization of the same formula:
    $$
    \underbrace{\xx \; \xx \; \xx}_{\text{input}} \;
    \underbrace{\orgate \; \texttt{\&11} \; \texttt{\&111}}_{\text{Or gate}} \;
    \underbrace{\notgate \; \texttt{\&111}}_{\text{Not gate}} \;
    \underbrace{\majoritygate \; \texttt{\&1} \; \texttt{\&1111} \; \texttt{\&11111}}_{\text{Majority gate}}
    $$

We now formalize the circuit evaluation problem, for any class of circuit families, such as the class $\TC^d$ of log$^d$-depth circuit families:

\begin{definition}[$\mathsf C$ Circuit Evaluation]
    \label{def:circuit-evaluation-problem}
    Let $\mathsf C$ be a class of (potentially non-uniform) circuit families.
    The $\mathsf C$ circuit evaluation problem is defined as follows:

    \begin{itemize}
        \item \underline{Input:} $(x, \langle C \rangle)$ where $x \in \{0, 1\}^*$ is a string and $\langle C \rangle$ is the serialization of a circuit $C$ such that $C = C_{\abs{x}}$ for some circuit family $\{C_n\}_{n=0}^\infty \in \mathsf C$.
        \item \underline{Output:} The value $C(x)$.
    \end{itemize}
\end{definition}

\CR{For example, the case where $\mathsf C = \Ppoly$ yields the generic \emph{circuit value problem}, which is known to be $\P$-complete. We focus here to the case of $\mathsf C = \TC^d$, i.e., the \emph{$\TC^d$ circuit evaluation problem}. It is somewhat intuitive that this problem is hard for the class $\AUniform\ \TC^d$ as long as $\A$ is strong enough to build circuits; we formalize this later in \cref{lem:circuit-evaluation-is-hard}. However, the $\TC^d$ circuit evaluation problem is \emph{not} necessarily in the class $\TC^d$.
%
To see this, suppose to the contrary that there exists an $\AUniform\ \TC^d$ circuit family $\mathcal C = \{C_n\}_{n=0}^{\infty}$ that solves the $\TC^d$ circuit evaluation problem. Then, for large enough $n$, each circuit $C_n \in \mathcal C$ has depth upper bounded by $c \log^d n$ for a \emph{fixed} $c > 0$. This $c$ being fixed is problematic---if one were to try to evaluate on string $x$ a circuit $C'_n$ from a $\TC^d$ circuit family that has depth $c' \log^d n$ where $c' > c$ (that is, invoke circuit evaluation for input $(x, \langle C'_n \rangle )$), the intuitive approach would require a circuit whose depth is larger than that of $C_n$.

To address this, we now formalize a \emph{constrained version} of the $\TC^d$ circuit evaluation problem that is, in fact, within the class $\FOUniform\ \TC^d$.
We achieve this via the \textbf{wide-$\TC^d$ circuit evaluation problem},
where the class wide-$\TC^d$ is defined as follows:

\begin{definition}[Wide Circuits]
    Let \textbf{wide-$\TC^d \subseteq \TC^d$} be the class of circuit families $\{C_n\}_{n=0}^\infty$ such that 
    there exists some $c$ such that, for large $n$, the depth of $C_n$ is at most $c \log^d n$ and, crucially, the size is \emph{at least} $n^c$.
\end{definition}

That is, wide-$\TC^d$ enforces that the size (and hence the width) of the circuit is large relative to its depth. In particular, for every wide-$\TC^d$ circuit family $\mathcal C$, there is a $c > 0$ such that the circuit $C_n$ has serialization of size $\Omega(n^c)$, and depth at most $c \log^d n$. Thus the depth of $C_n$ is at most $\log^d N$ where $N = n + \Omega(n^c)$ is the overall size of the input $(x, \langle C_n \rangle)$ to the corresponding circuit evaluation problem.
As we will show later, this allows the wide-$\TC^d$ circuit evaluation problem to be solved by a transformer (\cref{cor:circuit-evaluation-by-transformer}) as well as by an $\FOUniform\ \TC^d$ circuit family (\cref{lem:wide-TCd-eval-in-FOuniform-TCd}) using precisely $\log^d N$ iterations (manifested as loops of a transformer or $\FO$-uniform $\TC^0$ block, respectively), irrespective of the $\mathcal C$-dependent value of $c$.

Since it is a class of circuit families, wide-$\TC^d$ can be constrained by uniformity conditions in the natural way.
With some abuse of notation, we will use wide-$\TC^d$ to refer to both the class of circuit families and the complexity class of language classes the circuit families recognize.
}
Interestingly, this minimum size constraint imposed by wide-$\TC^d$ does not weaken it as a language class compared to $\TC^d$:
\begin{proposition}\label{prop:wideTCd-equals-TCd}
    For any $d \geq 0$, non-uniform wide-$\TC^d = \textrm{non-uniform}\ \TC^d$.
\end{proposition}

\begin{proof}
    Every wide-$\TC^d$ circuit family is, by definition, also a $\TC^d$ circuit family.
    
    Conversely, suppose $L \in \TC^d$. Then there is a circuit family $\{C_n\}_{n=0}^\infty$ recognizing $L$, where the depth of $C_n$ is at most $c \log^d n$ for some $c$ and large enough $n$. Consider a modified circuit family $\{C'_n \}_{n=0}^\infty$ where, for each $n$, $C'_n$ is a copy of $C_n$ that, if needed, is padded with dummy gates so that it has size at least $n^c$. This modified circuit family also recognizes $L$ but belongs to wide-$\TC^d$, completing the proof.
\end{proof}

In fact, the above equality holds even for uniform variants of these classes:
\begin{proposition}\label{prop:uniform-wideTCd-equals-uniform-TCd}
    For any $d \geq 0$ and $\A \supseteq \FO$, $\AUniform$ wide-$\TC^d = \AUniform\ \TC^d$.
\end{proposition}

\begin{proof}
    The proof follows that of \cref{prop:wideTCd-equals-TCd}.
    Every $\AUniform$ wide-$\TC^d$ circuit family is, by definition, also a $\AUniform$ $\TC^d$ circuit family.
    Conversely, suppose $L \in \A$-uniform wide-$\TC^d$.
    \CR{
    Then, we have a a circuit family $\mathcal C = \{C_n\}_{n=0}^\infty$ recognizing $L$, where the depth of each $C_n$ is at most $c \log^d n$ for some $c$ and large enough $n$.
    By $\A$ uniformity, we have the connection language $L_{\mathcal C} \in \A$.
    We define a circuit family $\mathcal C'$ that is $\mathcal C$ padded with dummy gates to make the number of gates at least $n^c$.
    The connection language $L_{\mathcal C'}$ defaults to $L_{\mathcal C}$ for gates that exist in $\mathcal C$ and simply returns a dummy gate for other gate indices up to $n^c$.
    This additional logic can be implemented in $\FO$, so $L_{\mathcal C'} \in \A \supseteq \FO$.
    Thus, the circuit family $\mathcal C'$ recognizes $L$ and is $\A$-uniform.
    }
\end{proof}


\subsection{Hardness of \texorpdfstring{Wide-$\TC^d$}{TCd} Circuit Evaluation}\label{sec:wide-TCd-hardness}

\begin{lemma} \label{lem:circuit-evaluation-is-hard}
    For $d \geq 0$, $\TC^d$ circuit evaluation is hard for $\LUniform$ $\TC^d$ under $\L$-reductions.
\end{lemma}

\begin{proof}

    Given any $L \in \LUniform\ \TC^d$, there exists a log-space Turing machine $T_L$ that constructs a circuit family $\{C_n\}_{n=0}^\infty$ that recognizes $L$. We can construct an $\L$-reduction from $L$ to the $\TC^d$ circuit evaluation problem as follows. Given an input $x$ whose membership in $L$ we would like to check, the reduction first copies $x$ to the output. It then uses the log-space Turing machine $T_L$ to build the circuit $C_{\abs{x}}$ and output it in the above serialized format.
    We thus have a log-space reduction from $x$ to $(x, \langle C_{\abs{x}} \rangle)$. We conclude that $\TC^d$ circuit evaluation is hard for $\LUniform$ $\TC^d$ under $\L$-reductions.
\end{proof}


\begin{lemma}\label{lem:wide-circuit-evaluation-is-hard}
    For $d \geq 0$, wide-$\TC^d$ circuit evaluation is hard for $\LUniform$ $\TC^d$ under $\L$-reductions.
\end{lemma}

\begin{proof}
As in the proof of \Cref{lem:circuit-evaluation-is-hard},
we are given $L$ such that there exists a log-space Turing machine $T_L$ that constructs $\{C_n \}_{n=0}^\infty$ recognizing $L$, where the depth of each $C_n$ is at most $c \log^d n$ for some $c$ and large enough $n$.
We can create a modified log-space Turing machine $T'_L$ that builds $\{C'_n \}_{n=0}^\infty$ that still recognizes $L$ in depth $c \log^d n$ but is padded, if needed, with dummy gates so that it has size at least $n^c$: we do this by keeping a counter for the number of gates and wires produced and outputting dummy ones until $n^c$ is exceeded.
We then follow the rest of the proof of \Cref{lem:circuit-evaluation-is-hard} with $T'_L$ instead of $T_L$.
\end{proof}


\subsection{Solving Wide-\texorpdfstring{$\TC^d$}{TCd} Circuit Evaluation with Transformers}\label{sec:solving-wide-TCd-with-transformers}

\begin{lemma}\label{lem:circuit-evaluation-by-transformer}
    There is a mixed-masked looped transformer $T$ that, on input $(x, \langle C \rangle)$ where $x \in \{0, 1\}^*$ and $\langle C \rangle$ is the serialization of a depth $\ell$ circuit with $\abs{x}$ inputs, computes $C(x)$ when unrolled $\ell$ times.
\end{lemma}

\begin{proof}
    We adapt the proof of \citet[Theorem 3]{merrill-sabharwal-2023-parallelism}, which sketches how log-depth (unlooped) transformers can implement the $\TC^0$ circuit evaluation problem. 
    We construct a looped transformer that will ``attempt'' to evaluate every gate: if its arguments have already been computed, the gate will return $\{0, 1\}$, and, if not, it will return undefined ($\bot$).

    Let $i$ be a token index.
    We say the token $w_i$ is a gate token if it is $\xx, \andgate, \orgate, \notgate, \majoritygate$. We will use gate token $i$ to store the value $v_i  \in \{0, 1, \bot\}$ for the gate it represents as a one-hot encoded vector.
    In the base case (embedding layer), we initialize $v_i = \bot$ for every gate token.
    Using a looped block of 2 layers, we will proceed in a way that propagates the computation of $v_i$ at later layers in terms of previously computed values.

    \paragraph{} \underline{$\xx$ Gates.} In the setup layers at an $\xx$ token $i$, we use a causally masked uniform attention head with value $\mathbbm{1}[w_j = \xx]$. Thus, this head computes $r_i/i$, where $r_i$ number of $\xx$ gates before and including $i$.
    We compute $\phi(r_i/i, 1/i) = \phi(r_i)$ and store it in the residual stream.

    In the looped layers, we define an attention head with query $\phi(r_i)$, key $\phi(j)$, and value $w_j$. This head thus retrieves input token $w_{r_j}$.
    We update the gate value to $v_i \gets w_{r_j}$ (viewing both as vectors in the same space).

    \paragraph{} \underline{Other Gates.} In the setup layers, each argument token $i$ attends with causally masked uniform attention with value $\mathbbm{1}[w_j = \&]$ to compute $a_i$, the number of arguments to its left (including it).
    Each \& token attends with query $\phi(a_i)$, key $\phi(a_j)$, and value $\mathbbm{1}[w_j = \&]$, which returns $1/(1 + z_i)$, where $z_i$ is the number of \texttt{1}'s following \& token $i$.
    We compute and store $\phi(z_i + 1)$ in the residual stream.
    We compute $g_i$ similarly to $a_i$, the number of gate tokens to left of token $i$ (also inclusive).

    In the first looped layer, each \& token $i$ attends with query $\phi(z_i + 1)$, key $\phi(g_j)$, and value $v_j$. Thus, the argument token $i$ retrieves $v_{z_i + 1}$, the value at gate $z_i + 1$.
    In the second looped layer, gate token $i$ attends with query $\phi(g_i)$, key $\phi(g_j)$, and value $v_{z_j + 1}$.
    This, it returns the vector $\frac{1}{\abs{A_i}} \sum_{j \in A_i} v_j$, where $A_i$ is the set of gates that are arguments of gate $i$.
    From this vector, we can apply projections to recover $T$, the fraction of $j \in A_j$ with $v_j = 1$, as well as $U$, the fraction of $j \in A_i$ with $v_j = \bot$.
    If $U > 0$, we set $v_i \gets \bot$.
    Otherwise, we set $v_i$ by thresholding $T$ against some threshold $k$ based on the gate type ($T \geq 1$ for \andgate, $T > 0$ for \orgate, and $T \geq 1/2$ for \majoritygate).
    In effect, this sets $v_i \gets G(\{v_j\}_{j \in A_j})$, where $G$ is the gate type.


    Thus, the looped layers either keep $v_i$ as $\bot$ or correctly update it to its true value.
    Furthermore, the number of looping steps until $v_i$ is updated is exactly the depth of node $i$.
    Thus, a circuit $C$ of depth $\ell$ can be fully evaluated by looping the depth-2 block $\ell$ times.
\end{proof}





\begin{lemcorollary}\label{cor:circuit-evaluation-by-transformer}
    For $d \geq 0$, wide-$\TC^d$ circuit evaluation is in $\mAHAT^d_0$, and hence in $\AHAT^d_1$.
\end{lemcorollary}

\begin{proof}
    We are given input $(x, \langle C_n \rangle)$, where $C_n$ comes from some wide-$\TC^d$ circuit family $\{C_n\}_{n=0}^\infty$ with depth at most $c \log^d n$ and size at least $n^c$, where $c$ is a constant specific to the circuit family.
    If we unroll the transformer from \Cref{lem:circuit-evaluation-by-transformer} to depth $c \log^d n$, we can solve wide-$\TC^d$ circuit evaluation problem for large enough $n$ (w.l.o.g.\ we can solve small-$n$ examples via table lookup).
    
    We next justify that a mixed-masked mAHAT$^d$ transformer will unroll at least $c \log^d n$ times for large enough $n$.
    This transformer will unroll exactly $\log^d N$ times, where $N = n + \abs{\langle C_n \rangle}$ is the total input length for a circuit evaluation instance.
    Since the size of $C_n$ is at least $n^c$, we have that $N \geq n^c$.
    Thus, our $\mAHAT^d_0$ transformer unrolls the following number of times:
    \begin{align*}
        \log^d N \geq \log^d n^c = c^d \log^d n \geq c \log^d n.
    \end{align*}
    Thus, our transformer will unroll a sufficient number of times to solve the wide-$\TC^d$ circuit evaluation problem. It follows that this problem is in $\mAHAT^d_0$.
    
    Finally, we conclude via \Cref{prop:mask-conversion} that this mixed-masked transformer can be converted to a corresponding causally masked padded transformer, placing the problem also in $\AHAT^d_1$.
\end{proof}


\subsection{Solving Wide-\texorpdfstring{$\TC^d$}{TCd} Circuit Evaluation with \texorpdfstring{$\FO$}{FO}-Uniform Circuits}\label{sec:solving-wide-TCd-with-circuits}

\CR{It is immediate that the looped transformer in \Cref{cor:circuit-evaluation-by-transformer} can be simulated in $\L$-uniform $\TC^d$. Moreover, since \Cref{lem:transformers-in-TCd} shows looped padded transformers can be simulated by $\FO$-uniform polylog-depth threshold circuit families, we obtain:


\begin{lemma}\label{lem:wide-TCd-eval-in-FOuniform-TCd}
    For $d \geq 0$, wide-$\TC^d$ circuit evaluation is in $\FOUniform\ \TC^d$.
\end{lemma}

Combined with \cref{lem:wide-circuit-evaluation-is-hard}, this implies the following completeness results for wide-$\TC^d$ circuit evaluation:
\begin{lemcorollary}\label{cor:wide-TCd-eval-complete-for-Luniform-TCd}
    For $d \geq 0$, wide-$\TC^d$ circuit evaluation is complete for $\LUniform\ \TC^d$ under $\L$ reductions.
\end{lemcorollary}

\begin{lemcorollary}\label{lem:wide-TCd-eval-complete-for-FOuniform-TCd}
    For $d \geq 0$, wide-$\TC^d$ circuit evaluation is complete for $\FOUniform\ \TC^d$ under $\L$ reductions.
\end{lemcorollary}
}




\section{Uniformity Collapse for Circuit Classes} \label{appendix:uniformity-collapse}

In general, for classes $\A, \B$ such that $\A \subseteq \B$, $\BUniformity$ often leads to larger classes of languages than $\AUniformity$, as we have more resources to construct a circuit. However, as the following lemma shows, this is not always the case:

\begin{proposition}[Uniformity Collapse] \label{prop:uniformity-collapse}
    Consider classes $\A, \B$ of functions and a class $\XC$ of polynomial size circuit families such that:
    \begin{enumerate}
        \item $\A \subseteq \B \subseteq \AUniform\ \XC$;
        \item $\mathsf{XC}$ circuit evaluation is in $\AUniform$ $\XC$;
        \item $\AUniform\ \XC$ circuit families are closed under fixed compositions (cf.~\Cref{sec:composition}).
    \end{enumerate}
    Then $\BUniformity$ does not strengthen $\AUniformity$,
    i.e., $\AUniform\ \XC = \BUniform\ \XC$.
\end{proposition}

\begin{proof}
    Since $\A \subseteq \B$, we trivially have $\AUniform\ \XC \subseteq \BUniform\ \XC$.
    The rest of the proof will focus on the other direction, showing that any language $L \in \BUniform\ \XC$ is also in $\AUniform\ \XC$.
    
    By the definition of $\BUniform\ \XC$, there exists a function $f \in \B$ that constructs an $\XC$ circuit family $\{C^f_n\}_{n \geq 0}$ that recognizes $L$. Specifically, $f(\texttt{1}^n) = C^f_n$ where $C^f_n$ is an $\XC$ circuit that checks membership in $L$ over all strings $w$ of size $n$: for $w \in \{0,1\}^n$, $C^f_n(w) = 1$ iff $w \in L$.
    The key insight is the following: while $\{C^f_n\}_{n \geq 0}$ is not necessarily an $\AUniform\ \XC$ circuit family, we can \emph{build} $C^f_n$ on demand using a different, $\AUniform\ \XC$ circuit family $\{C^g_n\}_{n \geq 0}$, by leveraging the first precondition of the lemma, specifically that $B \subseteq \AUniform\ \XC$. We will then leverage the second precondition to construct another $\AUniform\ \XC$ circuit family $\{C^h_{n,m}\}_{n,m \geq 0}$ that \emph{evaluates} the built circuit $C^f_n$ on the input string $w$. We will then use the third precondition to compose these two circuit families in order to obtain the final $\AUniform\ \XC$ circuit family $\{C^L_n\}_{n \geq 0}$ that recognizes $L$.
    
    \textbf{Building $C^f_n$:} Since $f \in \B$, by the first precondition of the lemma, $f$ is also in $\AUniform\ \XC$. We can thus simulate $f$ using an $\AUniform\ \XC$ circuit family.
    More concretely, there exists a function $g \in \A$ that, for any $n \geq 0$, maps input $\texttt{1}^n$ to an $\XC$ circuit $C^g_n$ that does what $f$ does, i.e., $C^g_n$ on input $\texttt{1}^n$ constructs the circuit $C^f_n$. Thus, we have $g(\texttt{1}^n) = C^g_n$ and $C^g_n(\texttt{1}^n) = C^f_n$.
    
    \textbf{Evaluating $C^f_n$ on $w$:} By the second precondition of the proposition, there exists a function $h \in \A$ that, for any $n, m \geq 0$, maps input $(\texttt{1}^n, \texttt{1}^m)$ to an $\XC$ circuit $C^h_{n,m}$ that evaluates any $\XC$ circuit $C_n$ of size $m$ over any input string $w$ of size $n$, i.e., $h(\texttt{1}^n, \texttt{1}^m) = C^h_{n,m}$ and $C^h_{n,m}(C_n, w) = C_n(w)$ for $\abs{w} = n$.
    
    \textbf{Composing builder and evaluator circuits:} Now we are ready to construct an $\AUniform\ \XC$ circuit family $\{C^L_n\}_{n \geq 0}$ that recognizes $L$. This circuit family, on input $w$ of length $n$, computes the following composition of the circuit builder family $\{C^g_n\}_{n \geq 0}$ and the circuit evaluator family $\{C^h_{n,m}\}_{n,m \geq 0}$:
    \begin{equation}\label{eqn:composed-circuit-on-w}
        C^L_n(w) = C^h_{n,m}(C^g_n(\texttt{1}^n), w)
    \end{equation}
    where $n = \abs{w}$ and $m = \abs{C^f_n}$. We first observe that this composite circuit does, in fact, recognize $L$. That's because for any input string $w$ of length $n$:
    \begin{align*}
        C^L_n(w)
          & = C^h_{n,m}(C^g_n(\texttt{1}^n), w) \\
          & = C^h_{n,m}(C^f_n, w) \\
          & = C^f_n(w)
    \end{align*}
    which, by definition, is $1$ if and only if $w \in L$. Thus, $C_L$ recognizes $L$.

    Lastly, by the third assumption of this proposition, the circuit family $\{C^L_n\}_{n \geq 0}$, being a fixed composition of $\AUniform$ circuit families $\{C^h_{n,m}\}_{n,m \geq 0}$, $\{C^g_n\}_{n \geq 0}$, and the identify function (for constructing the second argument $w$ of $C^h_{n,m}$), is also $\AUniform$.
\end{proof}

\Cref{prop:uniformity-collapse} may appear surprising on the surface level, but it has a quite intuitive interpretation: weaker uniformity provably cannot increase a circuit class that is sufficiently expressive.
This has an interesting implication for polylog-depth circuit classes: $\FO$ and $\LUniform$ $\TC^d$ collapse for $d \geq 1$:

\uniformitycollapsetcd*

\begin{proof}
    We state the proof for $\TC^d$, but everything generalizes for $\AC^d$ as well. By construction, the following containments hold:
    \begin{equation*}
        \FOUniform\ \TC^d \subseteq \LUniform\ \TC^d \subseteq \NLUniform\ \TC^d .
    \end{equation*}
    We will show that $\NLUniform\ \TC^d \subseteq \FOUniform\ \TC^d$ using \Cref{prop:uniformity-collapse}, with $\A = \FO$, $\B = \NL$, and $\XC = \text{wide-}\TC^d$.\footnote{This can be proven analogously for $\AC^d$ and wide-$\AC^d$.}
    With $\A = \FO$, the third precondition of \cref{prop:uniformity-collapse} follows from \cref{prop:fixed-composition}.
    We next argue that its first two preconditions are also met:\begin{enumerate}
        \item Since $d \geq 1$ and $\NL$ is known to be in $\FOUniform\ \AC^1$,\AS{CITATION?} which is contained in $\FOUniform\ \TC^1$, we have $\NL \subseteq \FOUniform\ \TC^d = \FOUniform\ \text{wide-}\TC^d$, where the final equality comes from \Cref{prop:uniform-wideTCd-equals-uniform-TCd}. Thus, the first precondition is satisfied.
        
        \item From \cref{lem:wide-TCd-eval-in-FOuniform-TCd}, we also have that $\text{wide-}\TC^d$ evaluation is in $\FOUniform\ \TC^d$, which by \cref{prop:uniform-wideTCd-equals-uniform-TCd} equals $\FOUniform\ \text{wide-}\TC^d$. Thus, the second precondition is also satisfied.
    \end{enumerate}
    We conclude by \Cref{prop:uniformity-collapse} that
    \begin{equation*}
        \NLUniform\ \TC^d \subseteq \FOUniform\ \text{wide-}\TC^d = \FOUniform\ \TC^d .
    \end{equation*}
    Thus, for $d \geq 1$, $\FOUniform$, $\LUniform$, and $\NLUniform$ $\TC^d$ collapse to the same class of languages.
\end{proof}

Notably, the proof of \Cref{thm:uniformity-collapse-ACd-TCd} does not go through in the case of $\TC^0$ since it is not known (and not believed) that $\L$ and $\NL$ are in $\TC^0$.

\newpage
\section*{NeurIPS Paper Checklist}

\begin{enumerate}

\item {\bf Claims}
    \item[] Question: Do the main claims made in the abstract and introduction accurately reflect the paper's contributions and scope?
    \item[] Answer: \answerYes{} 
    \item[] Justification: Main two results proved in \Cref{thm:tc0} and \Cref{thm:tck}.

\item {\bf Limitations}
    \item[] Question: Does the paper discuss the limitations of the work performed by the authors?
    \item[] Answer: \answerYes{} 
    \item[] Justification: Lack of empirical support is acknowledged and simplifications made to transformer architecture acknowledged in \Cref{sec:preliminaries}.

\item {\bf Theory Assumptions and Proofs}
    \item[] Question: For each theoretical result, does the paper provide the full set of assumptions and a complete (and correct) proof?
    \item[] Answer: \answerYes{} 
    \item[] Justification: All theorems are numbered, cross-referenced, and given full proofs. Relevant definitions are given in \Cref{sec:preliminaries}.

    \item {\bf Experimental Result Reproducibility}
    \item[] Question: Does the paper fully disclose all the information needed to reproduce the main experimental results of the paper to the extent that it affects the main claims and/or conclusions of the paper (regardless of whether the code and data are provided or not)?
    \item[] Answer: \answerNA{} 
    \item[] Justification: No experimental results.

\item {\bf Open access to data and code}
    \item[] Question: Does the paper provide open access to the data and code, with sufficient instructions to faithfully reproduce the main experimental results, as described in supplemental material?
    \item[] Answer: \answerNA{} 
    \item[] Justification: No experiments.

\item {\bf Experimental Setting/Details}
    \item[] Question: Does the paper specify all the training and test details (e.g., data splits, hyperparameters, how they were chosen, type of optimizer, etc.) necessary to understand the results?
    \item[] Answer: \answerNA{} 
    \item[] Justification: No experiments.

\item {\bf Experiment Statistical Significance}
    \item[] Question: Does the paper report error bars suitably and correctly defined or other appropriate information about the statistical significance of the experiments?
    \item[] Answer: \answerNA{} 
    \item[] Justification: No experiments.

\item {\bf Experiments Compute Resources}
    \item[] Question: For each experiment, does the paper provide sufficient information on the computer resources (type of compute workers, memory, time of execution) needed to reproduce the experiments?
    \item[] Answer: \answerNA{} 
    \item[] Justification: No experiments.
    
\item {\bf Code Of Ethics}
    \item[] Question: Does the research conducted in the paper conform, in every respect, with the NeurIPS Code of Ethics \url{https://neurips.cc/public/EthicsGuidelines}?
    \item[] Answer: \answerYes{} 
    \item[] Justification: No exceptional circumstances.

\item {\bf Broader Impacts}
    \item[] Question: Does the paper discuss both potential positive societal impacts and negative societal impacts of the work performed?
    \item[] Answer: \answerNA{} 
    \item[] Justification: This paper is foundational, focusing on general theoretical understanding of the computational power of transformers.
    
\item {\bf Safeguards}
    \item[] Question: Does the paper describe safeguards that have been put in place for responsible release of data or models that have a high risk for misuse (e.g., pretrained language models, image generators, or scraped datasets)?
    \item[] Answer: \answerNA{} 
    \item[] Justification: The paper poses no such risks.

\item {\bf Licenses for existing assets}
    \item[] Question: Are the creators or original owners of assets (e.g., code, data, models), used in the paper, properly credited and are the license and terms of use explicitly mentioned and properly respected?
    \item[] Answer: \answerNA{} 
    \item[] Justification: No artifacts used.

\item {\bf New Assets}
    \item[] Question: Are new assets introduced in the paper well documented and is the documentation provided alongside the assets?
    \item[] Answer: \answerNA{} 
    \item[] Justification: No artifacts released.

\item {\bf Crowdsourcing and Research with Human Subjects}
    \item[] Question: For crowdsourcing experiments and research with human subjects, does the paper include the full text of instructions given to participants and screenshots, if applicable, as well as details about compensation (if any)? 
    \item[] Answer: \answerNA{} 
    \item[] Justification: No human subjects.

\item {\bf Institutional Review Board (IRB) Approvals or Equivalent for Research with Human Subjects}
    \item[] Question: Does the paper describe potential risks incurred by study participants, whether such risks were disclosed to the subjects, and whether Institutional Review Board (IRB) approvals (or an equivalent approval/review based on the requirements of your country or institution) were obtained?
    \item[] Answer: \answerNA{} 
    \item[] Justification: No human subjects.

\end{enumerate}

\end{document}